\documentclass[a4paper,11pt]{article}
\usepackage{arxiv}

\usepackage[utf8]{inputenc} % allow utf-8 input
\usepackage[T1]{fontenc}    % use 8-bit T1 fonts
\usepackage{hyperref}       % hyperlinks
\usepackage{url}            % simple URL typesetting
\usepackage{booktabs}       % professional-quality tables
\usepackage{amsfonts}       % blackboard math symbols
\usepackage{amsmath}        % T added
\usepackage{amsthm}        % T added
\usepackage{xcolor}         % T added for comments
\usepackage{graphicx}       % T added for figs
\usepackage{subfigure}      % T added for figs
\usepackage{multirow}       % T added for multiple rows in table
\usepackage[capitalize]{cleveref}       % T added for better referencing
\usepackage{nicefrac}       % compact symbols for 1/2, etc.
\usepackage{microtype}      % microtypography
\usepackage{caption}

{\theoremstyle{plain}
\newtheorem{theorem}{Theorem}
}
{\theoremstyle{definition}

}

\usepackage{environ}
\newcommand{\acksection}{\section*{Acknowledgments}}
\NewEnviron{ack}{%
  \acksection
  \BODY
}

\usepackage{soul}

\newtheorem{corollary}[theorem]{Corollary}
\usepackage{algorithm}
\usepackage{algpseudocode}

\title{Unsupervised Learning of the Total Variation Flow} 
\author{Tamara G. Grossmann\thanks{Department of Applied Mathematics and Theoretical Physics, University of Cambridge, Cambridge, UK \hspace*{1.8em} (\texttt{\{tg410,sd870,cbs31\}@cam.ac.uk})} , Sören Dittmer\footnotemark[1]$\,$ \thanks{Center for Industrial Mathematics, University of Bremen, Bremen, Germany} , Yury Korolev\thanks{Department of Mathematical Sciences, University of Bath, Bath, UK (\texttt{ymk30@bath.ac.uk})} $\,$ and Carola-Bibiane Schönlieb\footnotemark[1]}

\usepackage{url}
\usepackage{hyperref}
\hypersetup{
    unicode=false,          % non-Latin characters in 
    colorlinks=true,       % false: boxed links; true: colored links
    linkcolor=red,          % color of internal links (change box color with linkbordercolor)
    citecolor=green,        % color of links to bibliography
    filecolor=magenta,      % color of file links
    urlcolor=blue           % color of external links
}

\usepackage{listings}
\usepackage{xspace}

\begin{document}
\maketitle

\vspace{20pt}
\begin{abstract}
The total variation (TV) flow generates a scale-space representation of an image based on the TV functional. This gradient flow observes desirable features for images, such as sharp edges and enables spectral, scale, and texture analysis. Solving the TV flow is challenging; one reason is the the non-uniqueness of the subgradients. The standard numerical approach for TV flow requires solving multiple non-smooth optimisation problems. Even with state-of-the-art convex optimisation techniques, this is often prohibitively expensive and strongly motivates the use of alternative, faster approaches. Inspired by and extending the framework of physics-informed neural networks (PINNs), we propose the TVflowNET, an unsupervised neural network approach, to approximate the solution of the TV flow given an initial image and a time instance. The TVflowNET requires no ground truth data but rather makes use of the PDE for optimisation of the network parameters. We circumvent the challenges related to the non-uniqueness of the subgradients by additionally learning the related diffusivity term. Our approach significantly speeds up the computation time and we show that the TVflowNET approximates the TV flow solution with high fidelity for different image sizes and image types. Additionally, we give a full comparison of different network architecture designs as well as training regimes to underscore the effectiveness of our approach.
\end{abstract}

\section{Introduction}
The total variation (TV) functional plays an important role in image processing. Introduced by Rudin, Fatemi, and Osher in 1992~\cite{RUDIN1992} for image denoising, it has since found successful applications in noise removal~\cite{RUDIN1992,Moeller2015}, image reconstruction~\cite{Wang2008}, and segmentation~\cite{Unger2008}, among many others. It is particularly suitable for image processing as it enforces piecewise constant regions and is edge-preserving. Minimising the TV functional through gradient descent yields the total variation flow~\cite{Andreu2001,Andreu2002,Bellettini2002}, a gradient flow evolving an image based on the subdifferential of the TV functional. The TV flow gives rise to spectral, scale, and texture analysis~\cite{Aujol2005,Brox2006}. In recent years, Gilboa introduced the spectral total variation decomposition~\cite{Gilboa2013,Gilboa2014} using the solution to the TV flow. The nonlinear spectral decomposition enables filtering and texture extraction at different scales based on the size and contrast of the structures in an image. The TV flow is the underlying PDE of the spectral TV decomposition. Applications of the decomposition include image denoising~\cite{Moeller2015}, image fusion~\cite{Benning2017,Zhao2018,Hait2019,liu2021multimodal}, segmentation for biomedical images~\cite{Zeune2017}, and texture separation and extraction~\cite{Brox2004,Horesh2016,Cohen2021}. There exists extensive theory on the TV flow~\cite{Andreu2001,Andreu2002,Bellettini2002,Caselles2013,Kinnunen2022}, as well as numerical studies~\cite{Feng2003,Breuss2006,Giga2019} and theory on the nonlinear spectral decomposition~\cite{Gilboa2013,Gilboa2014,Gilboa2018,Burger2016,bungert2019nonlinear}. However, obtaining the spectral TV decomposition of an image is computationally costly, mostly due to the need to solve the TV flow at every scale. Computing a solution to the TV flow is challenging because the subdifferential of TV is not a singleton unless the image has no constant regions. In this case, a subgradient of minimum norm must be chosen~\cite{bungert2019nonlinear}. Numerical methods either amount to modifying the gradient of the image in constant regions to make sure that the subdifferential is single-valued~\cite{Brox2004,Feng2003}, or they involve implicit schemes which requires solving multiple non-smooth optimisation problems~\cite{Gilboa2013,Gilboa2014}. The first option introduces artefacts, while the second one is computationally expensive, although work on improving its efficiency continues~\cite{Cohen2021,cohen:2021b}. 

We aim to leverage recent advances in applying deep learning to solve PDEs~\cite{Bar2019,E2018,Long2018,Lu2021,Raissi2019}. In particular, Raissi et al.~\cite{Raissi2019} introduced the physics-informed neural networks (PINNs) that approximate the PDE solution through a neural network. The classical model-driven approach to solving the TV flow without the need for smoothing requires solving multiple non-smooth optimisation problems, which is computationally expensive. Therefore, we propose a deep learning approach to approximate the TV flow solution for a given initial image and time instance. Previously, Getreuer et al.~\cite{Getreuer2021} learned the solution to the TV flow using the so-called BLADE network, emulating the Euler method. This supervised approach relies on having ground truth solutions of the TV flow through numerical approximations. It works as a time-stepping scheme, depending on the TV flow solution in previous instances.

\paragraph{Contributions}
This paper presents the TVflowNET, a neural network approach designed to approximate solutions to the TV flow problem. The approach is unsupervised, providing flexible and fast computation of TV flow solutions for arbitrary time instances and initial data. The main contributions of this work are as follows:
\begin{itemize}
    \item We introduce a novel energy functional that, when minimised, yields the solution to the entire TV flow of an image up to a specific time. The main difference between our approach and standard applications of PINNs is that our network learns the solution of the TV flow from an arbitrary initial image, as opposed to standard approaches that learn the solution of a PDE with fixed initial and boundary conditions. In other words, we learn a mapping from the product of the space of images (say, $L^2$) and time to the space of images;
    \item As an explicit form $\textnormal{div}\left( \nabla u / \lvert \nabla u \rvert \right)$ of the TV subgradient may be numerically unstable, we propose a loss functional that does not rely on this explicit form and instead uses the pointwise characterisation of the TV subgradient from~\cite{Bredies2016}. As a by-product, we also learn the subgradient at any time;
    \item In our numerical experiments, we first show that minimising the proposed loss functional via a joint space-time optimisation indeed approximates the solution to the TV flow by comparing it to time-stepping-based numerical methods - this step does not involve any neural networks; 
    \item We investigate the performance of the neural network approach on three different architecture designs and for four training regimes based on training images of different sizes. We evaluate the generalisability of these designs on images of different sizes and types. We demonstrate that the TVflowNET indirectly learns the behaviour of eigenfunctions (e.g. disks) of the subdifferential and retains TV flow properties, such as one-homogeneity for a certain set of contrast change factors;
    \item We show that the TVflowNET can successfully obtain the spectral TV decomposition, leveraging automatic differentiation in time to calculate the TV transform;
    \item Notably, we achieve a remarkable two orders of magnitude improvement in computation time compared to the model-driven approach. This significant reduction in processing time enhances the practicality and efficiency of TV flow solutions via TVflowNET in real-world applications.
\end{itemize}

This paper is structured as follows: First, we introduce the mathematical background of the TV flow and its classical numerical solution in Section~\ref{sec.maths_background}. We then present the deep learning methodology, the TVflowNET, to solve the TV flow in Section~\ref{sec.TVflowNET}. This includes the design of the loss functional and the network architectures we consider. We showcase our results for the joint space-time optimisation approach and the TVflowNET in Section~\ref{sec.results}. The latter is evaluated for generalisability and one-homogeneity, along with the performance for the spectral TV decomposition and an analysis of the computation time. Lastly, we give concluding remarks on the results and future perspectives.

\section{Mathematical Background}\label{sec.maths_background}
\begin{figure*}[t]
    \centering
    \includegraphics[width=\textwidth]{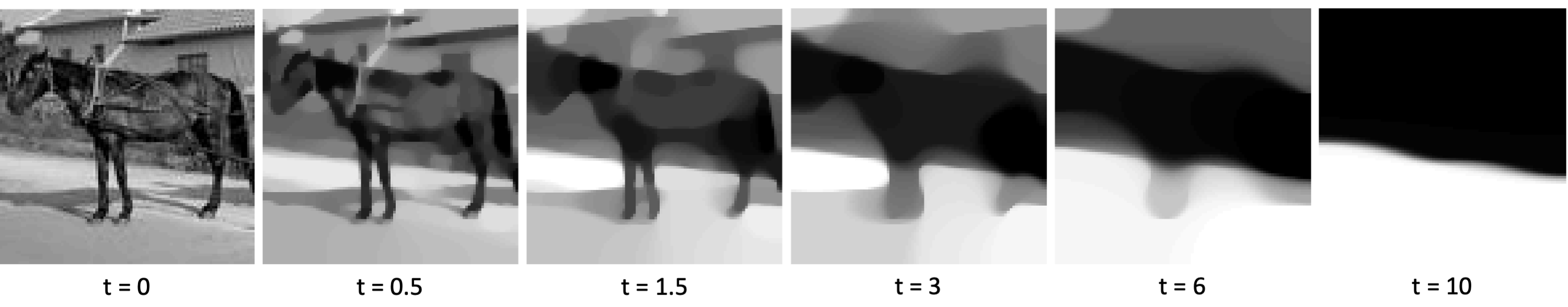}
    \caption[Example of the TV flow evolution.]{Example of the TV flow evolution. The TV flow solution at increasing time instances leads to more piecewise constant regions in the image. The initial image is taken from the STL-10 dataset~\cite{Coates2011}.}
    \label{fig.CH2.TVflow_horse_example}
\end{figure*}
\subsection{Theory}
Let us consider $\Omega \subset \mathbb{R}^2$ a bounded image domain with Lipschitz continuous boundary $\partial \Omega$. Following~\cite{chambolle:1997}, we can define the \emph{total variation} functional on $L^2(\Omega)$ as:
\begin{align}\label{eq.CH2.TV}
    J_{TV}(u) := 
    \begin{cases}
        \sup_{\lVert \varphi \rVert_{\infty} \leq 1} \int_{\Omega} u \, \textnormal{div} \, \varphi  \, \mathrm{d}x, \quad &\text{whenever this supremum is finite,} \\
        +\infty &\text{otherwise.}
    \end{cases} 
\end{align}
The supremum here is taken over the space of smooth compactly supported functions $\varphi := (\varphi_1,\varphi_2) \in C_0^{\infty}(\Omega; \mathbb R^2)$ and $\lVert \varphi \rVert_{\infty} := \sup_{x \in \Omega} \lVert \varphi(x) \rVert_2$. 

Minimising the TV functional~\eqref{eq.CH2.TV} through gradient descent yields the total variation flow~\cite{Andreu2001,Andreu2002,Bellettini2002}, a gradient flow evolving an image based on the subdifferential of the TV functional. The TV functional is non-smooth, and when differentiating with respect to $u$, we, therefore, need to consider the subdifferential defined as $\partial J_{TV}(u) = \{p \in L^2(\Omega): J_{TV}(v) \geq J_{TV}(u) + \langle p,v-u \rangle \,\, \forall v \in L^2(\Omega)\}$. If $\nabla u(x) \neq 0$, the  subdifferential is single-valued and given by~\cite{Bellettini2002,Gilboa2018}:
\begin{equation}\label{eq.CH2.TVsubgrad}
    \partial J_{TV}(u(x)) = \left\{-\text{div}\left( \frac{\nabla u(x)}{ \lvert\nabla u(x)\rvert_2} \right)\right\}, \quad \nabla u(x) \neq 0.
\end{equation}
If $\nabla u(x) = 0$ for some $x \in \Omega$, then $\partial J_{TV}(u)$ is set-valued. Let $u_0 \in L^2(\Omega)$ be an image. The \emph{TV flow}~\cite{Andreu2001} with Neumann boundary conditions is defined as: 
\begin{align}\label{eq.CH2.TVflow}
    \begin{cases}
    u_t(t;x) = -p(t;x), &p(t;x) \in \partial J_{TV}(u(t)), \,\, t \in (0,\infty), \,\, x \in \Omega,\\
    u(0;x) = u_0(x),  &x \in \Omega,\\
    \partial_{\vec{n}} u(t;x) = 0,  &t \in (0,\infty), \,\, x \in \partial\Omega,
    \end{cases}
\end{align}
The choice of Neumann boundary conditions is standard in imaging and motivated by the fact that it does not create an artificial jump at the boundary~\cite{Catte1992}. The TV flow generates a scale-space representation of an image based on the TV functional. It can be thought of as a diffusion process that enforces smoothing of the image while preserving discontinuities such as edges. An example of the TV flow evolution for an image is shown in Figure \ref{fig.CH2.TVflow_horse_example}. As the TV flow evolves, larger piecewise constant regions in the image will develop until the image becomes constant. The (finite) time at which this happens is called the extinction time. After this time $u_t(t;x) = 0$ and the solution is equal to the mean value of $u_0$~\cite{Andreu2002,Gilboa2018}.

There exists general theory of subgradient flows~\cite{brezis1973ope} of the form 
\begin{align}\label{eq.CH2.general_subgrad_flow}
    \begin{cases}
    u_t(t) + \partial J(u(t)) \ni 0 \quad t > 0,\\
    u(0) = u_0,
    \end{cases}
\end{align}
on a Hilbert space $\mathcal{H}$. The effective domain of the subdifferential is defined as $\text{dom}(\partial J) = \left\{ u \in \mathcal{H}: \partial J(u) \neq \emptyset \right\}$. We have the following existence result:

\newpage \begin{theorem}[Brezis {\cite[Thm. 3.2]{brezis1973ope}}]\label{thm.CH2.TVflow_existence}
Let $\mathcal{H}$ be a Hilbert space, $J:\mathcal{H} \to \mathbb{R} \cup \{+\infty\}$ a convex, lower semi-continuous and proper function, and $\partial J$ its subdifferential. For all $u_0\in\overline{\mathrm{dom}(\partial J)}$ there exists a unique function $u \in C([0,\infty);\mathcal{H})$ such that 
\begin{itemize}
\item $u(0)=u_0$,
\item $u(t) \in \mathrm{dom}(\partial J)  \,\, \forall t > 0$,
\item \eqref{eq.CH2.general_subgrad_flow} holds for almost every $t>0$,
\item $u(t)$ is Lipschitz continuous on $[\delta,\infty)$ for all $\delta > 0$ with
    \begin{align*}
        \lVert u_t \rVert_{L^{\infty}([\delta,\infty);\mathcal{H})} \leq \lVert \partial^0J(v) \rVert_{\mathcal{H}} + \frac{1}{\delta} \lVert u_0 - v \rVert_{\mathcal{H}} \quad \forall v \in \mathrm{dom}(\partial J), \, \forall \delta > 0,
    \end{align*}
    where $\partial^0J(u(t)):=\text{argmin}\left\lbrace\lVert p \rVert_{\mathcal{H}} \colon p\in\partial J(u(t))\right\rbrace\,\forall t>0$,
\item $u$ admits a right derivative for all $t > 0$ and
    \begin{align*}
        -\frac{d^+u}{dt}(t) = \partial^0J(u(t)) \quad \forall t >0,
    \end{align*}
\item $t\mapsto \partial^0J(u(t))$ is right continuous for all $t>0$ and $t\mapsto\lVert\partial^0J(u(t))\rVert_{\mathcal{H}}$ is non-increasing,
\item $t\mapsto J(u(t))$ is convex, non-increasing and Lipschitz continuous on $[\delta,\infty)$ for all $\delta > 0$ and 
    \begin{align*}
        \frac{d^+}{dt}J(u(t)) = -\Big\lVert \frac{d^+u}{dt}(t)\Big\rVert^2_{\mathcal{H}} \quad \forall t> 0.
    \end{align*}
\end{itemize}
\end{theorem}
The result has also been reproduced in the English-language, for example, in~\cite{bungert2019nonlinear}. For the TV functional defined on $L^2(\Omega)$~\eqref{eq.CH2.TV}, we have $\overline{\mathrm{dom}(J_{TV})} = L^2(\Omega)$~\cite{Bredies2016}. Since $L^2(\Omega)$ is a Hilbert space, we can apply Theorem~\ref{thm.CH2.TVflow_existence} to the TV flow and obtain existence and uniqueness results for the TV flow solution. We note that the definition of the TV functional on $L^2(\Omega)$ is only possible because $\Omega \subset \mathbb{R}^2$~\cite{ambrosio2000functions,chambolle:1997}. For existence and uniqueness results of the TV flow for $\Omega \subset \mathbb{R}^N$ with general $N\geq1$ we refer to~\cite{Bellettini2002}.

Let $\psi: L^2(\Omega)\times (0,\infty) \to L^2(\Omega)$ denote the solution operator to the TV flow~\eqref{eq.CH2.TVflow} with $\psi(u_0,t) := u(t)$ for $t\in (0,\infty)$ and $u_0 \in L^2(\Omega)$ the initial image in the PDE. For one-homogeneous functionals, such as the total variation, it can be shown that the subdifferential is invariant to contrast changes. That is, $\partial J_{TV}(u) = \partial J_{TV}(cu)$ for any constant $c>0$~\cite{Gilboa2014}. Multiplying the initial image by a constant $c>0$ results in a shift in the time variable:
\begin{align}\label{eq.CH2.one-hom}
    \psi(cu_0,t) = c\psi(u_0,t/c),
\end{align}
as shown in~\cite{Gilboa2013,Gilboa2014}. Let us also mention that because the TV flow only involves local operations, it is translationally equivariant. That is, shifting the initial image will result in the same shift of the image along the flow. In what follows, we will describe some known numerical approaches to solving the TV flow.

\subsection{Numerical Solution of the TV flow}
The numerical solution of the TV flow~\eqref{eq.CH2.TVflow} presents at least two major challenges. Firstly, the flow is non-linear. Secondly, the subgradient in~\eqref{eq.CH2.TVsubgrad} is numerically unstable in regions where $\nabla u(x)$ is small and not uniquely defined in constant regions. The second problem is often tackled by regularising the expression in~\eqref{eq.CH2.TVsubgrad} as follows
\begin{equation}\label{eq.CH2.TVsubgrad_reg}
    -\text{div}\left( \frac{\nabla u(x)}{ \sqrt{\lvert\nabla u(x)\rvert_2^2 +\epsilon^2 }}\right), \quad \epsilon>0.
\end{equation}
A straightforward approach would be to discretise the gradient with finite differences and to use explicit time-stepping schemes such as forward Euler to approximate the solution. However, explicit methods may suffer from stability issues and require small time steps to ensure convergence. One of the first methods that could handle the non-linearity of the TV flow~\eqref{eq.CH2.TVflow} in a stable manner is the semi-implicit lagged diffusivity scheme~\cite{chambolle:1997,chan1999convergence,Vogel1996}
\begin{equation}
    \frac{u(t+dt,x) -u(t,x)}{dt} = \text{div}\left( \frac{\nabla u(t+dt,x)}{ \sqrt{\lvert\nabla u(t,x)\rvert_2^2 +\epsilon^2 }}\right), \quad dt>0,
\end{equation}
which requires solving a linear elliptic equation at each iteration. Another semi-implicit scheme makes use of additive operator splitting and was proposed in~\cite{Weickert1998} to improve stability in time. A different strategy based on an auxiliary flux variable was proposed in~\cite{Chan1999a}, giving better convergence results and establishing a connection to primal-dual methods. Feng et al.~\cite{Feng2003} analyse the flow~\eqref{eq.CH2.TVflow} with a regularised subgradient~\eqref{eq.CH2.TVsubgrad_reg} as a minimal surface flow and propose to approximate the solution using the finite element method of the regularised problem. However, they require a certain degree of regularity and smoothing of the solution. Andreu et al.~\cite{Andreu2001} provide a semigroup analysis of the TV flow. An analysis of discretisation errors and oscillations in the TV flow, also with respect to regularised subgradient methods, was carried out in~\cite{Breuss2006}.  

A fully implicit time-stepping scheme that does not use the regularised subgradient~\eqref{eq.CH2.TVsubgrad_reg} and that is unconditionally stable in the step size was proposed in~\cite{Gilboa2013,Gilboa2014}. An implicit Euler step in the TV flow~\eqref{eq.CH2.TVflow} is equivalent to solving the following variational denoising problem known as the ROF problem~\cite{RUDIN1992}:
\begin{align}\label{eq.CH2.ROF}
    u(t + dt) = \textnormal{argmin}_{v} \lVert u(t) - v \rVert_2^2 + dt \, J_{TV}(v), \quad dt>0.
\end{align}
The Euler-Lagrange equation of this variational problem formulation~\eqref{eq.CH2.ROF} corresponds to an implicit Euler step of the TV flow. Efficient numerical schemes have been developed for solving~\eqref{eq.CH2.ROF}, we refer to the review~\cite{chambolle2016introduction} for details. In this paper, we will refer to solving~\eqref{eq.CH2.ROF} via Chambolle's projection method~\cite{Chambolle2004} as the \emph{model-driven approach}, which we run on the CPU. It works on the dual formulation of the ROF problem as an iterative algorithm that alternates between a gradient step on the dual variable and a projection step onto the relevant constraint set. We will also use a GPU implementation of the primal-dual hybrid gradient (PDHG) algorithm~\cite{Chambolle2010}. It alternates between a proximal step on the dual variable and a proximal step on the primal variable. For more details, we refer to~\cite{Chambolle2010,chambolle2016introduction}.

More recently, Cohen et al.~\cite{Cohen2021,cohen:2021b} proposed applying Koopman's theory of non-linear dynamical systems to the flow~\eqref{eq.CH2.TVflow} to improve computational efficiency. Their results are derived for the one-dimensional TV flow and in two dimensions only for anisotropic TV.

\section{Methods: TVflowNET}\label{sec.TVflowNET}
In the following, we present the TVflowNET, a deep learning approach for approximating the TV flow solution. Before deriving the TVflowNET, we collect a few basic properties of the TV flow and show that the solution operator is locally Lipschitz continuous. 

Brezis~\cite{brezis1973ope} investigated properties of general subgradient flows for subdifferentials of convex, lower semi-continuous and proper functions. We have summerised the results in Theorem~\ref{thm.CH2.TVflow_existence}. It can be easily shown that the theorem holds for the TV flow with the TV functional defined on $L^2(\Omega)$, which is a Hilbert space. Following these results, we get the existence of a unique solution which is continuous in time $\psi(\cdot,t) \in C([0,\infty);L^2(\Omega))$. The solution admits a right derivative in time that is equal to the subgradient of minimal norm, highlighting the importance of the subgradient of minimal norm for the solution of the PDE. In addition, Brezis showed Lipschitz continuity of the solution in $t \in [\delta,\infty)$ for $\delta >0$. 

In the following, we will show that the solution operator is locally Lipschitz continuous with respect to $u_0$ and $t$ jointly. We base our results on two theorems that show Lipschitz continuity in each variable independently. While Theorem~\ref{thm.CH2.TVflow_existence} showed Lipschitzness in time, its Lipschitz constant depends on $u_0$, which leads us to investigate local Lipschitzness in the two components. Lipschitz continuity with respect to the initial condition has been shown by Andreu-Vaillo et al.~\cite{Andreu-Vaillo2004parabolic} for $u_0 \in L^2(\Omega)$. Similar results have been published for $u_0 \in L^2(\mathbb{R}^N)$ \cite{Bellettini2002} and $u_0 \in L^1(\Omega)$ \cite{Andreu2001}. However, we focus on the $L^2(\Omega)$ case. 

\newpage \begin{theorem}[{\cite[Thm. 2.6]{Andreu-Vaillo2004parabolic}}] \label{thm.CH5.Lipschitz_u0}
    Let $u_0 \in L^2(\Omega)$. Then, the unique solution of~\eqref{eq.CH2.TVflow} is Lipschitz continuous in $u_0$. That is, let $\psi(u_0,t), \psi(\hat{u}_0,t)$ be solutions of~\eqref{eq.CH2.TVflow} corresponding to the initial conditions $u_0, \hat{u}_0 \in L^2(\Omega)$, then
    \begin{align*}
        \lVert \psi(u_0,t)- \psi(\hat{u}_0,t))\rVert_{L^2(\Omega)} \leq \lVert u_0 - \hat{u}_0\rVert_{L^2(\Omega)}, 
    \end{align*}
    for any $t \geq 0$.
\end{theorem}

Based on these two theorems, we can formulate a corollary about the local Lipschitz continuity of $\psi(u_0,t)$ with respect to $u_0$ and $t$ jointly. We show local Lipschitzness with respect to the 1-norm on the Cartesian product $L^2(\Omega) \times (0,\infty)$, i.e., $\lVert (u_0,t) \rVert_{1} := \lVert u_0 \rVert_{L^2(\Omega)} + \lvert t \rvert$. Because all norms on $\mathbb{R}^2$ are equivalent, this implies, of course, local Lipschitzness with respect to all other norms on $L^2(\Omega) \times (0,\infty)$ (i.e. combinations of $\lVert u \rVert_{L^2(\Omega)}$ and $\lvert t \rvert$).
\begin{corollary}\label{col.CH5.Lipschitz_jointly} 
    The solution operator $\psi(u_0,t)$ to the TV flow is locally Lipschitz continuous jointly in $(u_0,t)$, i.e. it is locally Lipschitz as an operator $L^2(\Omega) \times (0,\infty) \mapsto L^2(\Omega)$.
\end{corollary}
\begin{proof}
    Theorem~\ref{thm.CH2.TVflow_existence} states that for every $u_0 \in L^2(\Omega)$, the solution $\psi(u_0,t)$ is Lipschitz continuous on $[\delta,\infty)$ $\forall \delta > 0$ with a Lipschitz constant 
    \begin{align*}
        \lVert u_t \rVert_{L^{\infty}([\delta,\infty);L^2(\Omega))} \leq \lVert \partial^0J(v) \rVert_{L^2(\Omega)} + \frac{1}{\delta} \lVert u_0 - v \rVert_{L^2(\Omega)} \quad \forall v \in \mathrm{dom}(\partial J_{TV}), \, \forall \delta > 0.
    \end{align*}
    Taking $v = 0 \in \text{dom}(J_{TV})$ and $0 \in \partial J_{TV}(0)$, we have $\lVert u_t \rVert_{L^{\infty}(\delta,\infty;L^2(\Omega))} \leq \frac{\lVert u_0 \rVert_{L^2(\Omega)}}{\delta}$. Note that the Lipschitz constant is dependent on $u_0$. Similarly, we know that the solution to the TV flow is Lipschitz continuous in the initial condition with Lipschitz constant 1 following Theorem~\ref{thm.CH5.Lipschitz_u0}. 

    Let $u_0\in L^2(\Omega)$, $t\in (0, \infty)$, $0 < \epsilon < t$ and 
    \begin{align*}
        U_{\epsilon} := \left\{(v,\tau): \lVert v - u_0 \rVert_{L^2(\Omega)} < \epsilon, \lvert \tau - t \rvert < \epsilon \right\}.
    \end{align*}
    Then for all $(\bar{u}_0,\bar{t}), \, (\hat{u}_0,\hat{t}) \in U_{\epsilon}$ we have:
    \begin{align}\label{eq.CH5.corollary_proof}
        \begin{split}
        \lVert \psi(\bar{u}_0,\bar{t})- \psi(\hat{u}_0,\hat{t})\rVert_{L^2(\Omega)} &\leq \lVert \psi(\bar{u}_0,\bar{t})- \psi(\hat{u}_0,\bar{t})\rVert_{L^2(\Omega)} + \lVert \psi(\hat{u}_0,\bar{t})- \psi(\hat{u}_0,\hat{t})\rVert_{L^2(\Omega)} \\
        &\leq \lVert \bar{u}_0 - \hat{u}_0\rVert_{L^2(\Omega)} + \frac{\lVert \hat{u}_0 \rVert_{L^2(\Omega)}}{t-\epsilon} \lvert \bar{t}- \hat{t}\rvert\\
        &\leq \lVert \bar{u}_0 - \hat{u}_0\rVert_{L^2(\Omega)} + \frac{\lVert u_0 \rVert_{L^2(\Omega)} + \epsilon }{t-\epsilon} \lvert \bar{t}- \hat{t}\rvert,
        \end{split}
    \end{align}
    since $\lVert \hat{u}_0 \rVert_{L^2(\Omega)} \leq \lVert \hat{u}_0 - u_0 \rVert_{L^2(\Omega)} + \lVert u_0 \rVert_{L^2(\Omega)} \leq \epsilon + \lVert u_0 \rVert_{L^2(\Omega)}$. 
    Hence, we get
    \begin{align*}
        \lVert \psi(\bar{u}_0,\bar{t})- \psi(\hat{u}_0,\hat{t})\rVert_{L^2(\Omega)} \leq C_{\epsilon}(u_0,t)\lVert (\bar{u}_0,\bar{t}) - (\hat{u}_0,\hat{t}) \rVert_{1},
    \end{align*}
    with $C_{\epsilon}(u_0,t) = \max\left\{\frac{\lVert u_0 \rVert_{L^2(\Omega)} + \epsilon}{t-\epsilon}, 1 \right\}.$
\end{proof}

Now that we have collected some of the basic properties of the TV flow and its solution, we derive our numerical scheme. We have discussed numerical solutions of the TV flow using classical methods above. For the classical model-driven approach, the implicit Euler method requires solving the ROF problem~\eqref{eq.CH2.ROF} for each time step sequentially. Therefore, even if one is only interested in the result of the flow at time $T$, the entire flow up to time $T$ has to be evolved -- this can be computationally expensive and slow even with state-of-the-art optimisation approaches such as PDHG. In contrast, the inference speeds of trained deep learning approaches tend to be fast. Given an initial image $u_0$ and a time instance $t$, we propose to approximate the solution to the TV flow $u(t)$ via a neural network, the TVFlowNET. 

Our approach for solving the TV flow draws inspiration from PINNs~\cite{Raissi2019} which approximate the PDE solution through a neural network. We are particularly interested in the PINN methodology because it is unsupervised and does not require ground truth data. However, it requires modifications to solve our PDE of interest effectively. Let us summarise the building blocks to PINNs before we go into detail about the method itself. The first part comprises the design of the loss functional. Traditionally, this is the norm of the residual of the PDE and its initial and boundary conditions discretised on a random set of collocation points in the space-time domain. The architecture forms the second part of PINNs. The basic PINN version allows for broad neural network designs. While PINNs have been successfully used to approximate the solutions of various PDEs~\cite{Cai2021,Pang2019,Raissi2019}, they are not directly applicable to the TV flow for the following two reasons: Firstly, a PINN learns functions instead of operators, that is, it is trained for a fixed initial condition and requires retraining if initial or boundary conditions change. In contrast, we seek a neural network approximating the TV flow solution for any initial image. To overcome this, we use our loss functional to train a neural network that maps a given initial image, $u_0$, and a point in time, $t$, to the TV flow solution $u(t)$ in an unsupervised fashion. In defining the learned solution operator to be dependent on the initial condition, we avoid the need to retrain the neural network for new initial images. Instead, we get a solution operator for all possible images. Secondly, the explicit expression of the subgradient as $-\textnormal{div}\left( \nabla u / \lvert \nabla u \rvert \right)$ in the PDE residual is not feasible due to the singularities. We circumvent this challenge by deriving a new loss functional to learn the solution $u$ simultaneously with the subgradient $p \in J_{TV}(u)$. This requires a more precise characterisation of the subgradient. We do not have to regularise the TV subgradient with a small $\epsilon$ and avoid errors related to such smoothing. 

In the following, we present the details of the TVflowNET to approximate the TV flow solution. We introduce the loss functional and three different network architectures we train for the TVflowNET model. The performance of all three architectures will be evaluated and compared in the results section.

\subsection{Loss functional}
The first step in our deep learning approach is to build the loss functional. The functional will be minimised over both the TV flow solution and the subgradient of the TV functional simultaneously. To this end, we need a characterisation of the subdifferential $\partial J_{TV}(u)$. We follow the work by Bredies and Holler \cite{Bredies2016}:
\begin{theorem}[{\cite[Prop. 7]{Bredies2016}}] 
    \label{thm.CH5.subgrads}
    Let $\Omega \subset \mathbb{R}^2$, $u \in L^2(\Omega), p \in L^2(\Omega)$ and $J_{TV} \colon L^2(\Omega) \to \mathbb{R}\cup\{+\infty\}$ be as in \eqref{eq.CH2.TV}. Then $p \in \partial J_{TV}(u)$ if and only if
    \begin{align}
    \begin{cases}
        u \in \textnormal{BV}(\Omega) \,\, \text{and}\\[1pt]
        \exists \, \varphi \in \overline{C_0^{\infty}(\Omega, \mathbb{R}^2)}^{\lVert \cdot \rVert_{W^2(\textnormal{div})}} \,\, \text{with} \,\, \lVert \varphi \rVert_{L^{\infty}} \leq 1 \,\, \text{such that} \,\, p = -\textnormal{div} \varphi, \,\, \text{and} \\[1pt]
        J_{TV}(u) = - \langle u, \text{div} \, \varphi \rangle_{L^2(\Omega)}
    \end{cases}
    \end{align}
where $\lVert \cdot \rVert_{W^2(\textnormal{div})} := \lVert \cdot \rVert_{L^2}^2 + \lVert \textnormal{div}( \cdot) \rVert_{L^2}^2$ and $\lVert \varphi \rVert_{L^{\infty}} := \sup_{x \in \Omega} \lVert \varphi(x) \rVert_{L^2(\mathbb{R}^2)}$.
\end{theorem}
It should be noted that we consider the TV flow for data $u(t) \in L^2(\Omega)$ for all $t>0$. However, Theorem \ref{thm.CH5.subgrads} holds true for the more general case of $u \in L^q(\Omega)$ with $1 < q \leq \frac{d}{d-1}$ and $\Omega \in \mathbb{R}^d$. A characterisation of the subdifferential for $u \in L^1(\Omega)$ with similar results can be found in \cite{Andreu-Vaillo2004parabolic}. Based on the characterisation of the subdifferential for $u \in L^2(\Omega)$ as described in Theorem \ref{thm.CH5.subgrads}, we can equivalently rewrite the TV flow \eqref{eq.CH2.TVflow} with $p(t) = -\textnormal{div} \varphi(t)$ for $t \in (0,\infty)$ as follows:
\begin{align}\label{eq.CH5.TVflow_new}
    \begin{cases}
    u_t(t;x) = \textnormal{div} \, \varphi(t;x), & \textnormal{in} \,\, (0,\infty) \times \Omega\\[1pt]
    \lVert \varphi(t;\cdot) \rVert_{L^{\infty}} \leq 1,   &t \in (0,\infty) \\[1pt]
    J_{TV}(u(t;\cdot)) = - \langle u(t;\cdot), \text{div} \, \varphi(t;\cdot) \rangle_{L^2(\Omega)} & t \in (0,\infty) \\[1pt]
    u(0;x) = u_0(x),  &x \in \Omega\\[1pt]
    \partial_{\vec{n}} u(t;x) = 0,  &t \in (0,\infty), \,\, x \in \partial\Omega \\[1pt]
    \varphi (t;x) = 0 &t \in [0,\infty), \,\, x \in \partial\Omega\\[1pt]
    \partial_{\vec{n}} \varphi(t;x) = 0 &t \in [0,\infty), \,\,  x \in \partial\Omega
    \end{cases}
\end{align}
The boundary conditions for $\varphi$ reflect the fact that $\varphi(t) \in \overline{C_0^{\infty}(\Omega, \mathbb{R}^2)}^{\lVert \cdot \rVert_{W^2(\textnormal{div})}}$ and therefore both $\varphi$ and its gradient are zero at the boundary. As $\varphi$ is vector-valued, $\partial_{\vec{n}} \varphi(t;x) = 0$ is to be understood componentwise. The equivalent form \eqref{eq.CH5.TVflow_new} of the TV flow allows us to circumvent the challenges and errors related to the approximation of the subgradient, such as those arising from regularisation with $\epsilon$ in~\cite{chambolle:1997,Vogel1996}. Instead, we additionally learn the diffusivity term $\varphi$ as characterised in \eqref{eq.CH5.TVflow_new} via our neural network.

We will now formulate the loss functional. Our approach follows the idea of PINNs that minimise the residual of the PDE in the relevant norm. Given an initial image $u_0 \in L^2(\Omega)$ and a time instance $t \in [0, T]$ with $T>0$, let us denote $(u_{\theta}, \varphi_{\theta})$ as the TVflowNET output that approximates both the solution to the TV flow $u$ and the corresponding functional $\varphi$ in the subgradient $\text{div} \, \varphi$ of the TV functional. We can then define the loss functional as the residuals of~\eqref{eq.CH5.TVflow_new}, its initial conditions, and the minimum norm condition as measured by the $L^2$-norm:
\begin{align} \label{eq.CH5.loss}
    \begin{split}
    \mathcal{L}(\theta) = &\Big\lVert \frac{\partial u_{\theta}}{\partial t} - \text{div}\, \varphi_{\theta} \Big\rVert_{L^2([0,T]\times\Omega)}^2 + \alpha_1 \lVert \langle u_{\theta}, \text{div}\, \varphi_{\theta}\rangle + J_{TV}(u_{\theta})\rVert_{L^2([0,T])}^2 \\
    &+ \int_{\Omega}\int_{0}^T(\lVert \varphi_{\theta}(t;x) \rVert_{L^2(\mathbb{R}^2)} -1)_{+}\, \mathrm{d}t \, \mathrm{d}x + \lVert u_0 - u_{\theta}(0) \rVert_{L^2(\Omega)}^2 + \alpha_2 \lVert \text{div}\, \varphi_{\theta} \rVert_{L^2([0,T]\times\Omega)}^2,
    \end{split}
\end{align}
with $\alpha_1,\alpha_2 > 0$ constants, weighting the influence of the respective terms. The first three terms describe the residual of the TV flow~\eqref{eq.CH5.TVflow_new}, while the fourth loss term enforces the initial condition of the PDE. In addition to~\eqref{eq.CH5.TVflow_new}, we enforce the subgradient $\text{div} \varphi$ to be of minimal norm, reflected in the last term of the loss. Following the existence and uniqueness results in Theorem~\ref{thm.CH2.TVflow_existence}, it can be shown that the flow chooses a solution with subgradient of minimal norm~\cite{brezis1973ope,bungert2019nonlinear}. Overall, the loss describes how well the network outputs $( u_{\theta}, \varphi_{\theta})$ solve the TV flow.

One caveat of this learning approach is that the optimisation is typically gradient-based. This means we need to take the gradients with respect to the network parameters $\theta$ and inevitably also differentiate the TV functional using automatic differentiation.  Therefore, we need to add a small smoothing term $\epsilon >0$ to the TV functional to make it differentiable around $0$. We use the following smoothed TV functional:
\begin{align*}
	J_{TV_{\epsilon}}(u) = \int_{\Omega} \sqrt{\lVert \nabla u \rVert_{2}^2 + \epsilon}  \, \mathrm{d}x.
\end{align*}
We emphasise that the smoothing is only used for the TV functional in the term $\lVert \langle u_{\theta}, \text{div}\, \varphi_{\theta}\rangle + J_{TV}(u_{\theta})\rVert_{L^2([0,T])}^2$. We are not smoothing the subgradient $\text{div}\, \varphi$, but learning the diffusivity term. We confirmed numerically, that the learned diffusivity term is indeed not the diffusivity term in the smoothed subgradient, i.e., $\varphi \neq \nabla u /\sqrt{\lVert \nabla u \rVert_{2}^2 + \epsilon}$. The PSNR of those two terms averaged over all trained architectures is $10.98$dB and the SSIM is $0.107$, both very low especially in comparison to the similarities we report in the results section for learning the solution.

We enforce the spatial boundary conditions through the numerical derivation of the spatial derivatives via finite differences and the definition of the boundary therein. The divergence is the negative adjoint of the gradient. We ensure that the discretised gradient and divergence are adjoint by choosing forward differences for one and backward differences for the other and adjusting the boundary conditions of the differences accordingly (cf.~\cite{Chambolle2004}). We implement the Dirichlet condition on $\varphi$ as a hard constraint by setting the spatial boundary values of the network output $\varphi_{\theta}$ to 0. Lastly, we can evaluate the temporal derivative of $u_{\theta}$ via automatic differentiation. As time is an input to the neural network, we can use automatic differentiation, which essentially amounts to applying the chain rule multiple times through the network architecture. This means that no discretisation is needed to calculate the temporal derivative. Instead, we randomly sample in the time interval.

\subsubsection{Joint space-time optimisation}
We have designed the loss functional~\eqref{eq.CH5.loss} to optimise neural network parameters in an unsupervised learning approach. However, with a few modifications, the same loss can be used in a variational approach without the deep learning framework to approximate the TV flow solution. We can minimise the loss functional jointly in space and time on a fixed grid and given an initialisation of $u$ and $\varphi$. This joint space-time optimisation approach forms a baseline for the designed loss. 

Two main modifications facilitate the direct optimisation of the loss without using neural networks. Firstly, we discretise the temporal derivative with finite differences, as we cannot use automatic differentiation in this case. Secondly, we need to initialise the functions $u$ and $\varphi$ that the loss is minimised over. We initialise $u$ with the initial image $u_{ini} = u_0$ as evolving the TV flow through time will render increasingly piecewise constant versions of the initial image. The diffusivity term $\varphi$ is initialised with a regularised version $\varphi_{ini} = \nabla u / \left(\sqrt{\lVert\nabla u(t,x)\rVert_{L^2(\Omega)}^2 +\epsilon^2} \right)$.

The loss is optimised on the space and time grid simultaneously. That is, we do not run the optimisation for each time step separately as in the ROF problem~\eqref{eq.CH2.ROF}, but for a fixed number of pixels and time points. The optimisation is run for each initial image $u_0$ separately, and we expect this joint space-time optimisation to be computationally slow, especially compared to the TVflowNET approach. However, it does not serve the goal of a computational speedup on the TV flow solution but is used for comparison purposes to disentangle the effect of the variational problem and the neural network parametrisation.

\subsection{Architecture Design}
The second building block of our learning approach is the neural network architecture. The choice influences the ability of the network to approximate the PDE solution. To accommodate learning a solution operator for different initial data instead of function learning as in the vanilla PINNs method, the initial image $u_0$ will be part of the input to the neural network alongside the temporal variable $t$. We are looking at images; therefore, the grid in space is determined by the number of pixels of the image and, as opposed to PINNs, the spatial coordinates are not an input variable. In the following, we consider three different TVflowNET architecture designs that we will compare in their performance. 

\subsubsection{TVflowNET architectures}
For the design of the TVflowNET, we introduce three different neural network architectures. We consider what we call a Semi-ResNet, a U-Net and a learned gradient descent (LGD) approach. These neural networks differ in their complexity and the functions they employ in their network layers. In the results section, we will present a detailed comparison of the three networks based on their performance. 

Let us first consider the input and output format of the TVflowNET, which is the same for all network architectures. The input to our neural network is the initial image $u_0$ and the time instance $t$. The output represents the corresponding TV flow solution $u_{\theta}(t)$ and the diffusivity term $\varphi_{\theta}(t)$. We only consider convolutional neural networks (CNNs) in the design of the TVflowNET to enforce translational equivariance. CNNs have also been shown to perform well in imaging tasks. However, the input to a CNN needs to be of a form that enables convolution over the input. We, therefore, blow up the dimension of the time instance by copying the value of $t$ into a matrix to match the size of the initial image:
\begin{align*}
    \tau = \begin{pmatrix}
        t & \cdots & t\\
        \vdots & \ddots & \vdots \\
        t & \cdots & t
    \end{pmatrix} \in \mathbb{R}^{n \times m},
\end{align*} 
where $n \times m$ denotes the size of the initial image $u_0$. Subsequently, we concatenate the resulting matrix with $u_0$ to form the input $(u_0,\tau)$. Let us now describe the three network architectures in more detail.

For the first network architecture, we follow the TVspecNET~\cite{Grossmann2020} and DnCNN~\cite{Zhang2017} approaches that consider a ResNet~\cite{He2016}, a non-contracting feed-forward CNN that typically consist of network layers with skip connections. In this case, we use one skip connection: from the initial image $u_0$ to the $u_{\theta}(t)$ part of the output. In that, the network does not need to learn the entire image but only the residual of the solution, which is an increasingly piecewise constant version of the initial image with increasing time. The second output $\varphi_{\theta}(t)$ of the network uses no residual connection and, hence, we call the network a Semi-ResNet. We use eight sequential convolutional layers with 16 channels each. We chose these parameters heuristically to give the best results and balance network depth and computational cost. As activation functions, we use a combination of the ReLU and softplus function. While ReLU is more popular in image analysis applications~\cite{Grossmann2020,Zhang2017,Zheng2015}, we found it does not perform as well when used on its own. As a smoothed version of the ReLU function, the softplus function is differentiable and more stable~\cite{Zheng2015}. Given that the optimisation process of our loss functional involves second derivatives, this may explain the better performance of softplus compared to ReLU. However, softplus promotes images with smoothed edges; adding a single ReLU function in the neural network improved our results. We have done an extensive comparison of the use of a different number of ReLU versus softplus activations in this particular network. Numerical experiments showed that all networks containing at least one softplus activation perform similarly (PSNR standard deviation of 0.099). However, only having ReLU activations reduces performance by 2.39 PSNR points compared to the networks that use at least one softplus activation. With the highest attained performance, we consider the Semi-ResNet with ReLU in the first layer and softplus in the subsequent ones. 

The second architecture that we are considering is a U-Net~\cite{Ronneberger2015}. The U-Net architecture has been highly successful in training multiple imaging tasks. The U-Net has an encoder-decoder type architecture that uses down- and upsampling operations between its layers. In this way, the network uses pixel information from a larger neighbourhood and encodes hierarchical features. The idea is to learn fine-grid features in the higher levels and coarser features in the lower levels. In our case, the U-Net consists of three encoder and three decoder blocks, each equipped with softplus activation. Similarly to the Semi-ResNet architecture, we add a skip connection between the input image and the output solution $u_{\theta}(t)$. No residual connections are used for the diffusivity term $\varphi_{\theta}(t)$. 

\begin{algorithm}[t]

\begin{algorithmic}[1]
\caption{Learned Gradient Descent}\label{alg.CH5.LGD}
\vspace{2pt}
\State \textbf{Initialise:} $z_0 = u_0$, $\varphi_0 = \nabla u / (\sqrt{\lVert\nabla u\rVert_2^2+\epsilon^2})$
\State Let $g^i_{\theta}$ be the neural network block with weights $\theta$ that is trained to be the update step.
\For{i = 1, \dots, N} \Comment{$N=5$}
\State $\Delta u_{i+1},\, \Delta \varphi_{i+1},\, \Delta z_{i+1} = g^i_{\theta}(u_0,t, u_i, \varphi_i,z_i)$
\State $u_{i+1} = u_i + \Delta u_{i+1}$
\State $\varphi_{i+1} = \varphi_i + \Delta \varphi_{i+1}$
\State $z_{i+1} = z_i + \Delta z_{i+1}$
\EndFor
\State \textbf{return} $u_N, \varphi_N$
\vspace{2pt}
\end{algorithmic}
\end{algorithm}

Lastly, we consider a learned gradient decent-type network approach. LGD~\cite{Andrychowicz2016learning,flynn2019deepview} and similar approaches such as learned primal-dual~\cite{adler2018learned} change the learning task in that they learn the update steps of an optimisation approach. Let us assume we have an objective function $L(u)$, then the standard gradient descent reads $u_{i+1} = u_i + \alpha_i \nabla L(u_i)$. This update rule assumes differentiability of the objective function and does not include any second-order information. The idea of LGD is to learn the update step and replace it with neural network blocks for a fixed number of iterations. In our case, we train five network blocks, each representing one update step. Each block consists of two convolutional layers that are connected via a softplus activation. This particular set of parameter choices gave us the best results with the least amount of network depth for fast computation. The input of each block is the initial image $u_0$, the time matrix $\tau$, the previous updates of $u_{\theta}(t)$ and $\varphi_{\theta}(t)$, as well as the state of the update step that we denote by $z$. This additional variable has been shown to improve results as it directly links the state of the previous update step. The state variable is initialised with the initial image $u_0$ and the diffusivity term is initialised with the regularised version $\varphi_{ini} = \nabla u / \left(\sqrt{\lVert\nabla u\rVert_2^2 +\epsilon^2} \right)$. The output of each gradient block is then the learned update steps $\Delta u, \Delta \varphi, \Delta z$ that are added to the corresponding values. An overview of the LGD architecture is shown in Algorithm \ref{alg.CH5.LGD}. The network blocks $g^i_{\theta}$ with weights $\theta$ are then trained using the loss functional \eqref{eq.CH5.loss}. 

All three network architectures provide different advantages in their designs. The Semi-ResNet is a simple network with no down- and upsampling or different network blocks. With the chosen configuration, the Semi-ResNet has 14,947 parameters to be learned. We also choose this network design to investigate the complexity of architecture needed. It also gives a baseline as to the improvements through more complex network architectures. The U-Net itself has been proven beneficial in many imaging tasks and uses spatial down- and upsampling. However, a significantly larger number of parameters (272,972) needs to be trained compared to the Semi-ResNet. Lastly, the LGD has approximately twice as many parameters (33,415) as the Semi-ResNet. In the following, we will compare the performance of these three networks and investigate the generalisability as well as the training and evaluation time.

\section{Results}\label{sec.results}
The goals of our numerical experiments are as follows: First, we investigate results for minimising the designed loss functional \eqref{eq.CH5.loss} without a deep learning framework and, in that, evaluate the feasibility of the loss. These results serve as a baseline and show that the minimiser of the loss functional approximates a solution to the TV flow. We expect this to be a slow optimisation, and it does not serve the goal of a computational speedup of the TV flow solution. 

Secondly, we evaluate the neural network approach, that is, the TVflowNET, in detail. We show all results for the three different architecture designs and implement training regimes on four different image sizes. We investigate the generalisability of the trained networks by looking into the performance of the TVflowNET on images of different sizes and types compared to the training set, such as medical images. 

One of the properties of the TV functional is one-homogeneity. We investigate which of the networks can retain this property and how well. As the transform in the spectral TV decomposition is based on the second temporal derivative of the TV flow solution, we also look into the use the TVflowNET in the derivation of spectral TV-filtered images. That is, we compare the feasibility of using the automatic differentiation through the neural network for calculating the second temporal derivative of $u$ compared to finite differences. Lastly, we evaluate the computation time, that is, the training time and the evaluation time and compare them to the model-driven approach and, for the sake of completeness, also to the joint space-time optimisation.

In the following, we will use the model-driven approach as our ground truth solution. By model-driven approach, we refer to the repeated solution of the non-smooth ROF problem~\eqref{eq.CH2.ROF} via Chambolle's projection method~\cite{Chambolle2004} as introduced by Gilboa et al.~\cite{Gilboa2013,Gilboa2014}. While the original spectral TV decomposition paper~\cite{Gilboa2014} provides code that runs on a CPU, we are additionally using a primal-dual implementation~\cite{Chambolle2010,Hammernik2017} that runs on a GPU. Since the neural network training and evaluation typically is done on GPU, this allows for a comparison based on similar resources. We use greyscale images for simplicity.

\subsection{Joint space-time optimisation}
The optimisation for the joint space-time approach is run for each initial image $u_0$ individually but simultaneously for $50$ equidistant time points in $[0,1]$. We employ the Adam optimiser with a learning rate of $5e-3$. The loss term weighting factors are set to $\alpha_1, \alpha_2 = 0.0001$, which provided empirically the best results to approximate the TV flow solution $u$. We evaluate the approach on $200$ images of the Food101~\cite{Food101} test dataset. We crop the images to evaluate three different image sizes, that are, $32\times32$px, $96\times96$px, and $256\times256$px. For images of size $32 \times 32$px and $96 \times 96$px, we iterate for $2,000$ epochs. In comparison, we need significantly more iterations for larger images such as $256 \times 256$px to obtain good results; that is, we iterated for $14,000$ epochs. 

\begin{table}[t]
\centering
    \begin{tabular}{ccccc}
     \toprule
    & & \multicolumn{3}{c}{\textbf{Evaluation Size}}                   \\
     & & 32 x 32px & 96 x 96px & 256 x 256px  \\
     \hline
     $u$ & PSNR &  $42.05 \pm 3.03$ & $43.47 \pm 1.88$ & $40.39 \pm 0.27$\\
     & SSIM &  $0.989 \pm 0.008$ & $0.990 \pm 0.006$ & $0.978 \pm 0.0005$\\
     $\text{div}\varphi$ &PSNR &  $29.86 \pm 1.47$ & $29.61 \pm 1.40$ & $24.09 \pm 0.95$\\
     & SSIM &  $0.671 \pm 0.142$ & $0.676 \pm 0.111$ & $0.4075 \pm 0.137$\\
     \bottomrule
    \end{tabular}
    \caption[Evaluation of the joint space-time optimisation compared to the model-driven approach for solving the TV flow $u$ and approximating the diffusivity term $\varphi$.]{Evaluation of the joint space-time optimisation compared to the model-driven approach for solving the TV flow $u$ and approximating the diffusivity term $\varphi$ at 50 time instances on $200$ images of the Food101~\cite{Food101} testing dataset for three different image sizes. Reported PSNR and SSIM values correspond to the average and standard deviation over the dataset and time instances.\label{tb.CH5.joint_space-time}}
\end{table}

We evaluate the performance of the joint space-time approach for approximating both the TV flow solution $u$ and the diffusivity term $\varphi$. The model-driven approach serves as the ground truth solution for comparison. The results based on the image similarity measures PSNR and SSIM are shown in Table \ref{tb.CH5.joint_space-time}. For the TV flow solution $u$, we can show that the joint space-time optimisation is able to approximate the solution images with high fidelity with a maximum PSNR of 43.477dB and SSIM of 0.99. We also report the standard deviation over the evaluation dataset which shows the result is relatively consistent over the dataset. While the approach produces good approximations of the TV flow solution with fewer iterations for smaller images, obtaining a similar quality for larger images requires a significantly higher number of iterations, increasing the computation time. However, we can show that the PSNR and SSIM values also consistently increase with an increasing number of iterations. For $256 \times 256$px images, the optimisation with $2,000$ iterations rendered a PSNR of $30.51$dB, with $7,000$ iterations of $36.30$dB and with $10,000$ iterations of $38.74$dB. We additionally ran the optimisation for $14,000$ iterations as reported in Table~\ref{tb.CH5.joint_space-time}, which rendered a PSNR of $40.39$dB. However, as we will later look at the computation time, there is a trade-off to be made between image quality and the time it takes to run the optimisation. For this purpose, we compare the joint space-time optimisation approach for $14,000$ iterations and not higher for the $256 \times 256$px images. An example image of the result is shown in Figure \ref{fig.CH5.TVFlowNET_result_u}. 

The model-driven approach does not directly calculate the diffusivity term $\varphi$. Therefore, we evaluate the ability of the joint space-time optimisation approach to approximate $\varphi$ by comparing the divergence $\text{div}\varphi$ from the joint space-time optimisation to the finite difference approximation of the temporal derivative $u_t$ of the model-driven approach. Compared to the PSNR and SSIM values for the TV flow solution $u$, the metric results for $\text{div}\varphi$, as shown in Table \ref{tb.CH5.joint_space-time}, are considerably lower. We attribute this result to the fact that the weighting of the terms in the loss functional, including the minimal norm factor, was chosen in terms of the performance of the TV flow solution $u$. Visual inspection of the resulting images as displayed in Figure \ref{fig.CH5.TVFlowNET_result_phi} shows that while the PSNR and SSIM values are lower compared to the evolved TV flow images, the joint space-time optimisation still recovers the divergence with good accuracy, only differing from the ground truth for larger time instances. Overall, the experiments show the ability of our loss functional \eqref{eq.CH5.loss} to attain good reconstructions of TV flow images while accurately representing the temporal evolution of the TV flow. As a result, it can be considered a suitable baseline for the deep learning approach. We use this approach to see what we can expect from the neural network approximation.

\subsection{TVflowNET}
\begin{table}[t]
\centering
    \begin{tabular}{ccccc}
     \toprule
    \textbf{Architecture} & \textbf{Training Size} & \multicolumn{3}{c}{\textbf{Evaluation Size}}                   \\
          & & 32 x 32px & 96 x 96px & 256 x 256px  \\
     \hline
     Semi-ResNet & 32 x 32px & 41.07 $\pm$ 0.09 & 42.73  $\pm$ 0.12 & 43.59  $\pm$ 0.15\\
     Semi-ResNet & 96 x 96px & 39.52 $\pm$ 0.23 & 41.48  $\pm$ 0.22 & 42.39  $\pm$ 0.34\\
     Semi-ResNet & 256 x 256px & 29.25 $\pm$ 1.05 & 32.38  $\pm$ 1.07  & 33.85  $\pm$ 1.02\\
     Semi-ResNet & mixed & 39.35 $\pm$ 0.05 & 40.93  $\pm$ 0.12 & 41.42  $\pm$ 0.21\\
     U-Net & 32 x 32px & \textbf{42.2 $\pm$ 0.08} & \textbf{43.95   $\pm$ 0.21} & 44.84  $\pm$ 0.31\\
     U-Net & 96 x 96px & 41.48 $\pm$ 0.21 & 43.86  $\pm$ 0.19 & 44.82  $\pm$ 0.2\\
     U-Net & 256 x 256px & 38.18 $\pm$ 0.14 & 39.63  $\pm$ 0.25 & 40.1  $\pm$ 0.29\\
     U-Net & mixed & 41.16  $\pm$ 0.3 & 42.08  $\pm$ 0.29 & 42.29  $\pm$ 0.28\\
     LGD & 32 x 32px & 41.72  $\pm$ 0.08 & 43.8  $\pm$ 0.12 & \textbf{45.14  $\pm$ 0.16}\\
     LGD & 96 x 96px & 39.95  $\pm$ 0.45 & 42.27  $\pm$ 0.29 & 43.6  $\pm$ 0.2\\
     LGD & 256 x 256px & 35.15  $\pm$ 0.13 & 38.27  $\pm$ 0.19 & 39.55  $\pm$ 0.28\\
     LGD & mixed & 39.79  $\pm$ 0.17 & 41.66  $\pm$ 0.09 & 42.67  $\pm$ 0.12\\
     \bottomrule
    \end{tabular}
    \caption[Evaluation (PSNR) of the TVflowNET trained on different architectures and image sizes for solving the TV flow $u$.]{Evaluation (PSNR) of the TVflowNET trained on different architectures and image sizes for solving the TV flow $u$ at 50 time instances for three different image sizes. The reported PSNR values are based on the comparison to the model-driven approach and correspond to the average over $200$ images of the Food101~\cite{Food101} testing dataset.  The best result for each evaluation size is marked in bold. We additionally report the standard deviation based on the network trained five separate times for each architecture.}\label{tb.CH5.TVflowNET_results_u_psnr}
\end{table}

\begin{table}[t]
\centering
    \begin{tabular}{ccccc}
     \toprule
    \textbf{Architecture} & \textbf{Training Size} & \multicolumn{3}{c}{\textbf{Evaluation Size}}                   \\
          & & 32 x 32px & 96 x 96px & 256 x 256px  \\
     \hline
     Semi-ResNet & 32 x 32px & 0.986  $\pm$ 0.0002 & 0.989  $\pm$ 0.0002 & 0.989  $\pm$ 0.0002 \\
     Semi-ResNet & 96 x 96px & 0.982  $\pm$ 0.0008  & 0.986  $\pm$ 0.0007  & 0.987  $\pm$ 0.0009 \\
     Semi-ResNet & 256 x 256px & 0.89  $\pm$ 0.018 & 0.939  $\pm$ 0.009 & 0.95  $\pm$ 0.006\\
     Semi-ResNet & mixed & 0.981  $\pm$ 0.0005 & 0.985  $\pm$ 0.0004 & 0.985  $\pm$ 0.0005\\
     U-Net & 32 x 32px & \textbf{0.991  $\pm$ 0.0003} & 0.992  $\pm$ 0.0007 & 0.991  $\pm$ 0.0008\\
     U-Net & 96 x 96px & 0.988  $\pm$ 0.0009 & 0.992  $\pm$ 0.0003 & 0.992  $\pm$ 0.0005\\
     U-Net & 256 x 256px & 0.974  $\pm$ 0.0008 & 0.981  $\pm$ 0.001 & 0.982  $\pm$ 0.001\\
     U-Net & mixed & 0.987  $\pm$ 0.001 & 0.988  $\pm$ 0.0009 & 0.986  $\pm$ 0.0008\\
     LGD & 32 x 32px & 0.989  $\pm$ 0.0002 & \textbf{0.992  $\pm$ 0.0002} & \textbf{0.993  $\pm$ 0.0002}\\
     LGD & 96 x 96px & 0.985  $\pm$ 0.001 & 0.99  $\pm$ 0.0005 & 0.991  $\pm$ 0.0004\\
     LGD & 256 x 256px & 0.966  $\pm$ 0.001 & 0.981  $\pm$ 0.001 & 0.983  $\pm$ 0.001\\
     LGD & mixed & 0.984  $\pm$ 0.0004 & 0.989  $\pm$ 0.0003 & 0.989  $\pm$ 0.0003\\
     \bottomrule
    \end{tabular}
    \caption[Evaluation (SSIM) of the TVflowNET trained on different architectures and image sizes for solving the TV flow $u$.]{Evaluation (SSIM) of the TVflowNET trained on different architectures and image sizes for solving the TV flow $u$ at 50 time instances for three different image sizes. The reported SSIM values are based on the comparison to the model-driven approach and correspond to the average over $200$ images of the Food101~\cite{Food101} testing dataset.  The best result for each evaluation size is marked in bold. We additionally report the standard deviation based on five separate network trainings of each architecture.}\label{tb.CH5.TVflowNET_results_u_ssim}
\end{table}
The TVflowNET is trained for three different network architectures, as introduced above. We run four training regimes for each of the architectures based on different image sizes. The networks are trained on 625 images cropped to sizes $32 \times 32$px, $96 \times 96$px, $256 \times 256$px and finally on 625 images of mixed sizes randomly cropped between $32 \times 32$px and $256 \times 256$px of the Food101 dataset~\cite{Food101}. During training, we randomly sampled the time instances from a uniform distribution over the interval $[0,1]$ in each epoch, and we use automatic differentiation to evaluate the temporal derivative in the loss functional. We calculate spatial derivatives via finite differences. 

We use the Adam optimiser to minimise the loss functional. All networks are trained for $750$ epochs with a learning rate of $3e-4$ on a Nvidia A100 GPU. We note that the training time is limited to 36 hours due to the high-performance computing restrictions of the available resources. In the case of training on $256 \times 256$px images, the time limit is reached before running 750 epochs. We choose the loss term weighting factors $\alpha_1, \alpha_2 = 0.0001$, which have empirically shown the best results to approximate the TV flow solution $u$ in training.

\begin{figure}[t]
    \centering
    \includegraphics[width=\textwidth]{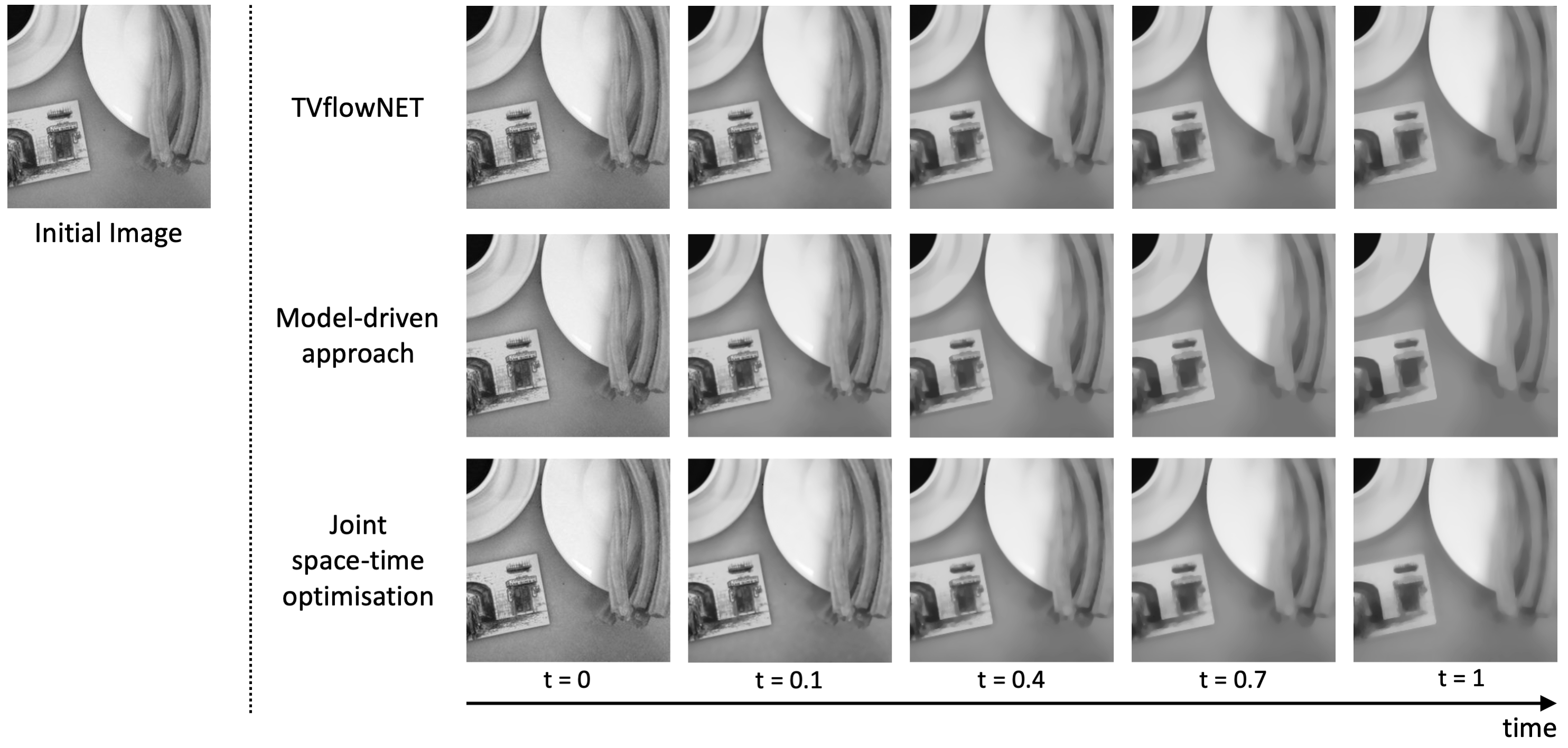}
    \caption[Visualisation of the TV flow solution $u$ derived using the TVflowNET, the model-driven approach and the joint space-time optimisation.]{Visualisation of the TV flow solution $u$ derived using the TVflowNET with the U-Net architecture, the model-driven approach and the joint space-time optimisation. The example image is taken from the Food101~\cite{Food101} dataset.}
    \label{fig.CH5.TVFlowNET_result_u}
\end{figure}

We investigate the influence of the training size on the performance of the TVflowNET applied to test images of different sizes. We evaluate the TVflowNET against the model-driven approach with test images from the Food101 dataset~\cite{Food101} with image sizes $32 \times 32$px, $96 \times 96$px and $256 \times 256$px. For comparison, we have evaluated the network on 50 equidistant time instances in $[0,1]$. The results for the quantitative similarity measures PSNR and SSIM evaluated on the testing dataset of 200 images for the TV flow solution $u$ are shown in Tables \ref{tb.CH5.TVflowNET_results_u_psnr} (PSNR) and \ref{tb.CH5.TVflowNET_results_u_ssim} (SSIM). All architectures and training regimes are able to approximate the TV flow solution at a high quality. All except one network consistently have PSNR values above 35dB and an SSIM of 0.96 or higher. The best performance is observed for the U-Net and LGD trained on $32 \times 32$px images with a PSNR value ranging up to $45.14$dB and a maximum SSIM of $0.993$. We additionally report the standard deviation with respect to the network training. That is, we trained each network five separate times. Overall, we can see a low standard deviation for all architectures, indicating that our approach is able to reliably produce the reported results.

We make several observations across all network architectures, considering the different training and testing regimes. The networks trained on $256 \times 256$px images have the lowest performance results independent of the evaluation image size or architecture. We attribute this behaviour to the fact that the training time was restricted to 36 hours by the resources available. Therefore, the optimisation may not have fully converged. The best performance for this network is shown on testing images of the same size as the training images (i.e., $256 \times 256$px) and shows that is does not generalise to other image sizes as well as the networks trained on smaller image sizes did.

However, training on larger images is unnecessary to obtain high-quality results for those image sizes. The networks trained on $32 \times 32$px and $96 \times 96$px images show outstanding performance on $256 \times 256$px images. More so, the PSNR values are higher than for the smaller testing images, whereas the SSIM stays approximately constant between the different image sizes for evaluation. This may be due to the dependency of the PSNR on the mean squared error (MSE): The TV flow evolution leads to more piecewise constant regions in the image while filtering out smaller, lower contrast structures. Therefore, when we consider large structures in the larger testing images, they may remain unchanged in the TV flow evolution up to time $t=1$, resulting in a low MSE and, hence, a higher PSNR. For our trained TVflowNET, the size and the contrast of the structures that are being filtered out through the TV flow evolution are independent of the image size in the training set as the time interval we train on stays fixed at $[0,1]$. 

Overall, while still delivering high-quality results, the Semi-ResNet has the lowest performance compared to the other architectures which is no surprise as this architecture has the least amount of parameters. The U-Net and LGD, both being more complex in their design and having more parameters than the Semi-ResNet, perform similarly well on all image sizes. An example result for a $256 \times 256$px image of the TV flow solution $u$ generated by the U-Net network trained on $32 \times 32$px is shown in Figure \ref{fig.CH5.TVFlowNET_result_u}.

\begin{table}[t]
\centering
    \begin{tabular}{ccccc}
     \toprule
    \textbf{Architecture} & \textbf{Training Size} & \multicolumn{3}{c}{\textbf{Evaluation Size}}                   \\
          & & 32 x 32px & 96 x 96px & 256 x 256px  \\
     \hline
    Semi-ResNet & 32 x 32px & 31.94 (0.821) & 34.01 (0.867) & 34.52 (0.881) \\
    Semi-ResNet & 96 x 96px & 31.09 (0.771) & 33.37 (0.831) & 34.01 (0.848) \\
    Semi-ResNet & 256 x 256px & 22.59 (0.191) & 25.45 (0.242) & 26.39 (0.257) \\
    Semi-ResNet & mixed & 30.64 (0.721) & 32.44 (0.773) & 32.76 (0.788) \\
    U-Net & 32 x 32px & \textbf{32.62} (0.836) & 34.7 (0.873) & 35.16 (0.885) \\
    U-Net & 96 x 96px & 32.36 (0.824) & 34.7 (0.873) & 35.23 (0.889) \\
    U-Net & 256 x 256px & 29.6 (0.475) & 31.09 (0.538) & 31.12 (0.552) \\
    U-Net & mixed & 31.46 (0.744) & 33.11 (0.77) & 33.21 (0.763) \\
    LGD & 32 x 32px & 32.53 \textbf{(0.854)} & \textbf{34.89 (0.896)} & \textbf{35.62 (0.91)} \\
    LGD & 96 x 96px & 31.81 (0.823) & 34.16 (0.875) & 34.86 (0.889) \\
    LGD & 256 x 256px & 28.26 (0.546) & 30.59 (0.621) & 31.0 (0.641) \\
    LGD & mixed & 31.27 (0.794) & 33.54 (0.846) & 34.15 (0.863) \\
     \bottomrule
    \end{tabular}
    \caption[Evaluation of the TVflowNET trained on different architectures and image sizes for approximating $\text{div}\,\varphi$.]{Evaluation of the TVflowNET trained on different architectures and image sizes for approximating $\text{div}\,\varphi$ at 50 time instances for three different image sizes. The reported PSNR (SSIM) values are based on the comparison to the model-driven approach and correspond to the average over $200$ images of the Food101~\cite{Food101} testing dataset.\label{tb.CH5.TVflowNET_results_phi}}
\end{table}

For the purpose of completeness, we evaluate the approximation of the diffusivity term $\varphi$. As in the joint space-time optimisation evaluation, we consider the approximation of $\text{div} \varphi$ and compare it to the finite difference approximation of the temporal derivative $u_t$ of the model-based approach. The evaluation protocol is the same as for the TV flow solution approximation. We derive the quality measures on the same testing dataset of different image sizes and for the same set of time instances. The results are shown in Table \ref{tb.CH5.TVflowNET_results_phi}. Again, the average PSNR and SSIM are lower than the results for the TV flow solution $u$. As in the joint space-time optimisation, the loss functional and, more specifically, the weights therein are chosen to give the best performance in terms of approximating the TV flow solution and not necessarily the diffusivity term. However, the quality measure values are higher than the joint space-time optimisation. This is also visually confirmed by an example image of $\text{div}\varphi$ in Figure \ref{fig.CH5.TVFlowNET_result_phi}. The best results are attained by the U-Net and LGD trained on $32 \times 32$px images with similar performance. 

\begin{figure}[t]
    \centering
    \includegraphics[width=\textwidth]{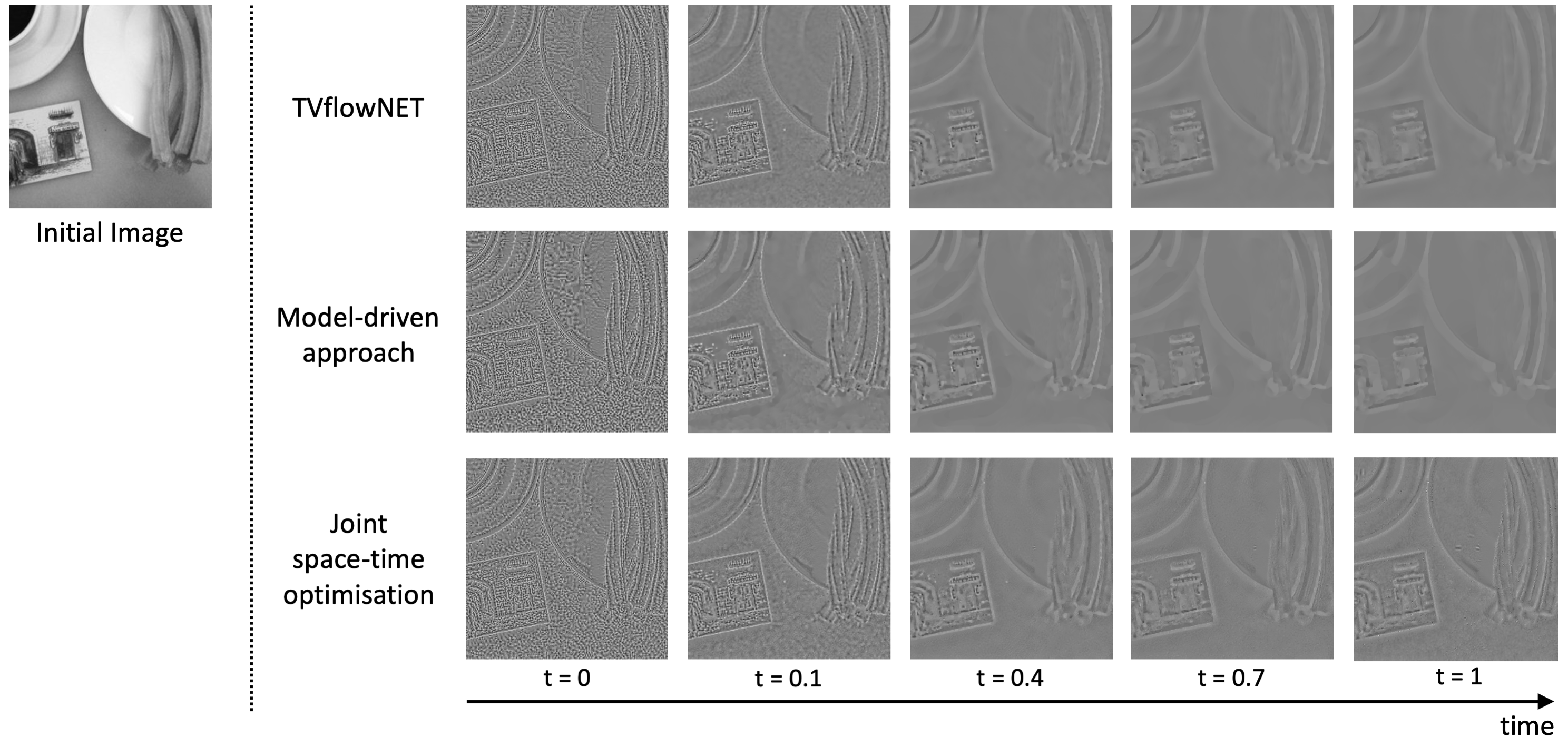}
    \caption[Visualisation of $\text{div}\,\varphi$ derived using the TVflowNET, the model-driven approach and the joint space-time optimisation.]{Visualisation of $\text{div}\,\varphi$ derived using the TVflowNET with the U-Net architecture, the model-driven approach and the joint space-time optimisation. The example image is taken from the Food101~\cite{Food101} dataset.}
    \label{fig.CH5.TVFlowNET_result_phi}
\end{figure}

\subsubsection{Generalisability}
\begin{table}[t]
\centering
    \begin{tabular}{ccccc}
     \toprule
    \textbf{Architecture} & \textbf{Training Size} & \multicolumn{3}{c}{\textbf{Evaluation Data Type}}                   \\
          & & Disks & Ellipses & Medical Data  \\
     \hline
     Semi-ResNet & 32 x 32px & \textbf{52.12} (0.983) & 43.32 (0.947) & 44.83 (0.984)\\
     Semi-ResNet & 96 x 96px & 43.92 (0.798) & 40.8 (0.906) & 43.35 (0.959)\\
     Semi-ResNet & 256 x 256px & 33.07 (0.856) & 30.07 (0.782) & 35.22 (0.935)\\
     Semi-ResNet & mixed & 42.77 (0.786) & 39.21 (0.906) & 42.46 (0.956)\\
     U-Net & 32 x 32px & 51.11 (0.975) & 43.62 \textbf{(0.963)} & \textbf{46.36 (0.989)}\\
     U-Net & 96 x 96px & 49.98 (0.968) & 43.74 (0.947) & 45.8 (0.985)\\
     U-Net & 256 x 256px & 40.96 (0.826) & 36.87 (0.883) & 41.55 (0.946)\\
     U-Net & mixed & 49.59 (0.964) & 41.0 (0.917) & 44.08 (0.976)\\
     LGD & 32 x 32px & 51.14 (0.963) & \textbf{44.62} (0.963) & 45.98 (0.987)\\
     LGD & 96 x 96px & 51.57 \textbf{(0.984)} & 43.4 (0.94) & 45.04 (0.983)\\
     LGD & 256 x 256px & 37.36 (0.604) & 34.14 (0.766) & 40.95 (0.937)\\
     LGD & mixed & 43.52(0.867) & 37.82 (0.902) & 44.02 (0.974)\\
     \bottomrule
    \end{tabular}
    \caption[Evaluation of the TVflowNET trained on different architectures and image sizes for solving the TV flow $u$ for differing image types.]{Evaluation of the TVflowNET trained on different architectures and image sizes for solving the TV flow $u$ at 50 time instances for differing image types, i.e., images of disks and ellipses as well as medical images. The reported PSNR (SSIM) values are based on the comparison to the model-driven approach and correspond to the average over $100$ images for the disk and ellipse images and over $40$ images of the spleen segmentation dataset~\cite{MedSegDecathlon}.\label{tb.CH5.TVflowNET_results_disks}}
\end{table}
We have already demonstrated the generalisability of the networks to images of different sizes than the training set (cf. Tables \ref{tb.CH5.TVflowNET_results_u_psnr} and \ref{tb.CH5.TVflowNET_results_u_ssim}). In the following, we further examine the performance of the networks on images of different types, specifically disks, ellipses, and medical images. 

The first two categories are based on the importance of eigenfunctions of the TV flow. Indicator functions of disks and ellipses are eigenfunctions of the subdifferential of the TV functional. In the spectral TV decomposition, these eigenfunctions constitute the basic elements and produce an isolated impulse response that is directly related to the size and contrast of the structure~\cite{Gilboa2014}. The shape of such structures remains unchanged throughout the TV flow evolution. However, the contrast gradually diminishes. To illustrate this effect, Figure \ref{fig.CH5.TVFlowNET_result_disk_plot} exhibits an example of the linearly decreasing pixel values for two disks differing in size and contrast. We investigate whether the learned mapping preserves this behaviour on eigenfunctions. The results are presented in Table \ref{tb.CH5.TVflowNET_results_disks}. Our evaluation encompasses $100$ images comprising isolated disks of different sizes and contrasts, along with $100$ images of overlapping ellipses with varying contrasts. The results show that all architectures of the TVflowNET successfully approximate the TV flow solution for these images, even though they were trained on natural images that did not feature isolated disks or overlapping ellipses. Notably, the performance on disk images surpasses that of the networks on natural images. The networks have thus implicitly learned the intrinsic properties of the TV flow and, consequently, the behaviour of the basic elements of the decomposition, with no direct guidance or having encountered similar images. The performance of TVflowNET on ellipse images aligns with its performance on natural images. To provide visual examples, Figure \ref{fig.CH5.TVFlowNET_result_disk} presents the results for a disk image, while Figure \ref{fig.CH5.TVFlowNET_result_ellipse} displays the outcomes for an ellipse image.

\begin{figure}[t]
         \centering
         \includegraphics[width=\textwidth]{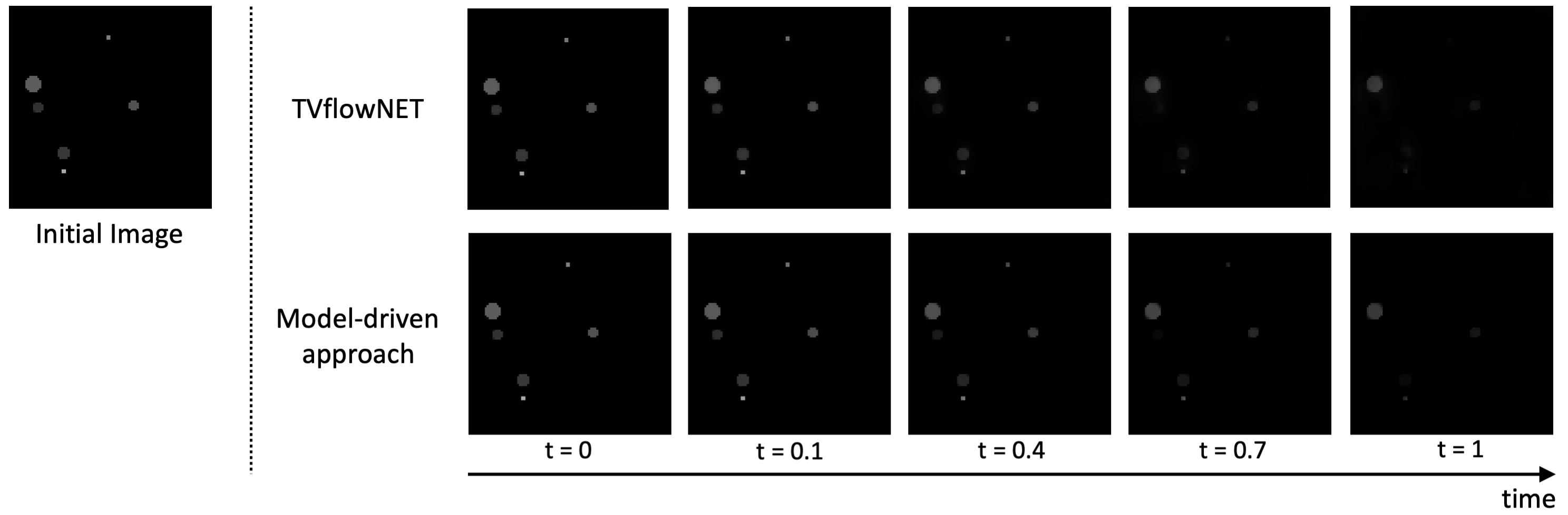}
         \caption[TV flow solution of an image containing isolated disks comparing the TVflowNET and the model-driven approach.]{TV flow solution of an image containing isolated disks derived using the TVflowNET with the U-Net architecture and the model-driven approach.}
         \label{fig.CH5.TVFlowNET_result_disk}
\end{figure}

\begin{figure}[t]
         \centering
         \includegraphics[width=0.85\textwidth]{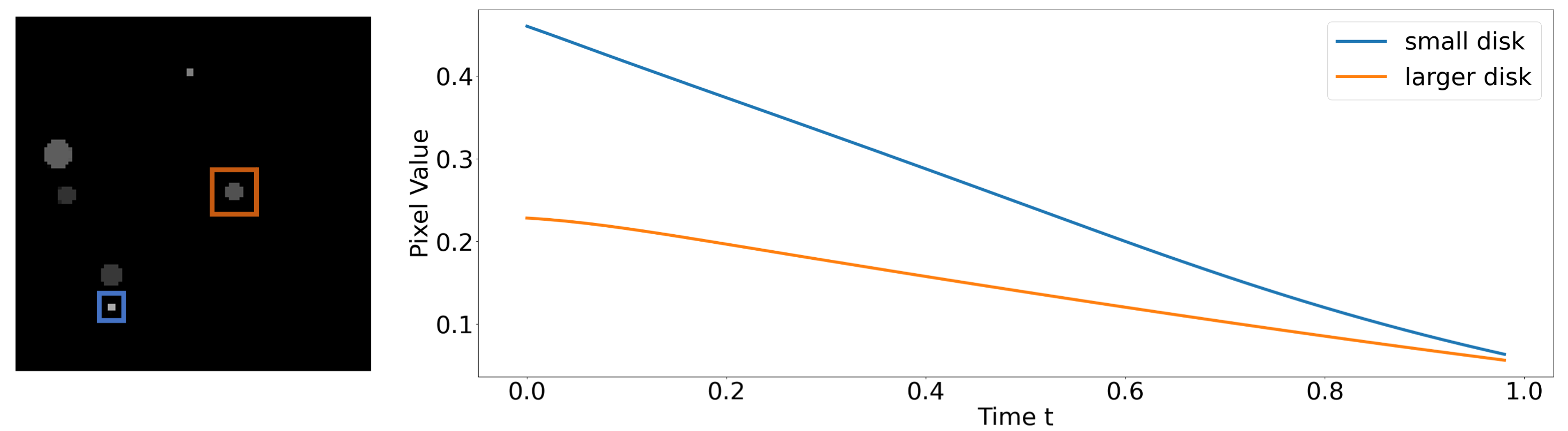}
         \caption[Evolution of the pixel value through the TV flow for two of the disks.]{Evolution of the pixel value through the TV flow for two of the disks. Pixel values decrease linearly.}
         \label{fig.CH5.TVFlowNET_result_disk_plot}
\end{figure}

\begin{figure}[t]
    \centering
    \includegraphics[width=\textwidth]{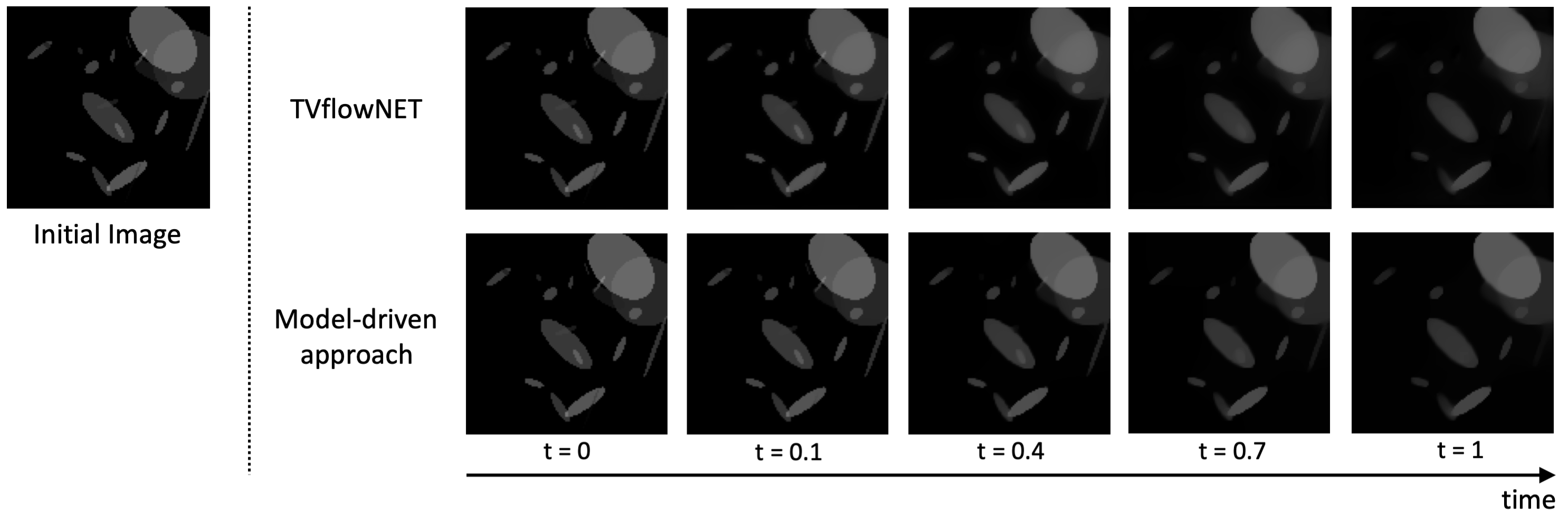}
    \caption[Example of the TV flow solution for ellipse images comparing the TVflowNET and the model-driven approach.]{Example of the TV flow solution for ellipse images derived using the TVflowNET with the U-Net architecture and the model-driven approach.}
    \label{fig.CH5.TVFlowNET_result_ellipse}
\end{figure}

Lastly, we consider medical images of the Medical Segmentation Decathlon Dataset~\cite{MedSegDecathlon}. Medical images differ from natural images as they often contain smooth, dark backgrounds with defined structures that represent parts of the scanned body. In this case, we evaluate the TVflowNET performance on the data depicting spleen scans. The image resolution is $512 \times 512$px. The results are shown in Table \ref{tb.CH5.TVflowNET_results_disks} and an example image is shown in Figure \ref{fig.CH5.TVFlowNET_result_spleen}. All network architectures are able to approximate the TV flow solution with high accuracy with a PSNR of up to $46.36$dB. Indeed, the experiments on disks, ellipses and medical data confirm that the trained TVflowNET is able to approximate the TV flow solution successfully with high fidelity.

\begin{figure}[t]
    \centering
    \includegraphics[width=\textwidth]{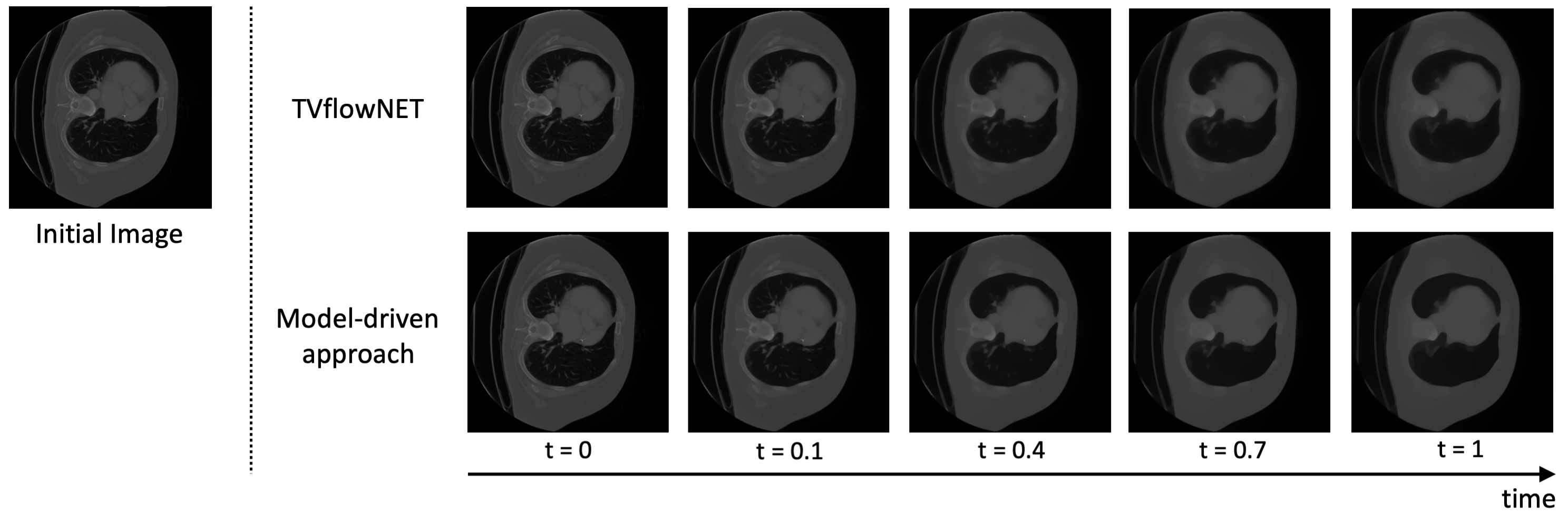}
    \caption[Visualisation of the TV flow solution for medical data comparing the TVflowNET and the model-driven approach.]{Visualisation of the TV flow solution for medical data derived using the TVflowNET with the U-Net architecture and the model-driven approach. The example image is taken from the Medical Segmentation Decathlon Dataset~\cite{MedSegDecathlon}.}
    \label{fig.CH5.TVFlowNET_result_spleen}
\end{figure}

\subsubsection{One-homogeneity}
\begin{table}[t]
\centering
    \begin{tabular}{cccc}
     \toprule
    \textbf{Architecture} & \textbf{Training Size} & \multicolumn{2}{c}{\textbf{Factor for 1-hom. evaluation}}                   \\
          & & 3 & 0.3  \\
     \hline
     Semi-ResNet & 32 x 32px & 51.22 (0.997) & 29.91 (0.852) \\
     Semi-ResNet & 96 x 96px & 46.71 (0.996) & 32.4 (0.929) \\
     Semi-ResNet & 256 x 256px & 28.55 (0.944) & 22.35 (0.839) \\
     Semi-ResNet & mixed & 40.45 (0.994) & 28.94 (0.813)\\
     U-Net & 32 x 32px & 50.34 (0.997) & \textbf{35.78} (0.941) \\
     U-Net & 96 x 96px & 46.02 (0.994) & 35.41 (0.952) \\
     U-Net & 256 x 256px & 35.24 (0.983) & 34.36 \textbf{(0.957)} \\
     U-Net & mixed & 39.67 (0.988) & 34.68 (0.943) \\
     LGD & 32 x 32px & \textbf{52.0 (0.998)} & 26.11 (0.721) \\
     LGD & 96 x 96px & 42.27 (0.993) & 29.12 (0.849) \\
     LGD & 256 x 256px & 30.6 (0.979) & 28.57 (0.879) \\
     LGD & mixed & 43.53(0.992) & 10.38 (0.646) \\
     \bottomrule
    \end{tabular}
    \caption[Evaluation of the one-homogeneity property of the TVflowNET trained on different architectures and image sizes at 50 time instances for three different image sizes.]{Evaluation of the one-homogeneity property of the TVflowNET trained on different architectures and image sizes at 50 time instances for three different image sizes. The reported PSNR (SSIM) values are based on comparison to the model-driven approach and correspond to the average over $600$ images of different sizes of the Food101~\cite{Food101} testing dataset.\label{tb.CH5.TVflowNET_results_one-hom}}
\end{table}
One of the properties of the TV flow is derived from the one-homogeneity of the TV functional. That is, contrast changes of the image result in a shift of the time parameter in the TV flow solution~\eqref{eq.CH2.one-hom}. We evaluate the TVflowNET on its ability to retain the one-homogeneity property without direct penalisation in training. To this end, we multiply the test images with two different factors, $c = 3$ and $c = 0.3$ and compare them to the TVflowNET applied to the original images and evaluated at times $t/3$ and $t/0.3$, respectively. The results are based on $600$ images of mixed resolution and are shown in Table~\ref{tb.CH5.TVflowNET_results_one-hom}. For the factor $c = 3$, we obtain high PSNR and SSIM values with a maximum of $52$dB in the LGD network trained on $32 \times 32$px images. Due to the relation between contrast change and time evolution, the resulting times for $t \in [0,1]$ correspond to $t/3 \in [0,0.33]$ and with that lie in the originally trained time interval. In contrast, for initial images multiplied with $c = 0.3$, the time shift corresponds to an interval of $[0,3.33]$ and some of the network evaluations lie well outside the trained time parameters. Therefore, the one-homogeneity property also investigates the generalisability to larger time instances of the trained TVflowNET. The results in Table \ref{tb.CH5.TVflowNET_results_one-hom} show that in this case the networks do not retain the one-homogeneity property with high accuracy. Solely the U-Net architecture has PSNR values consistently above $34$dB. However, for the larger time instances, this also becomes less accurate. We show an example for the U-Net trained on $32 \times 32$px images in Figure \ref{fig.CH5.TVFlowNET_result_onehom}. For images corresponding to $t/0.3>1$, we can see a degradation of the results. Therefore, the one-homogeneity property holds for the TVflowNET only for factors $c>0$ such that $t/c \in [0,1]$ in the trained time interval. 

\begin{figure}[t]
    \centering
    \includegraphics[width=\textwidth]{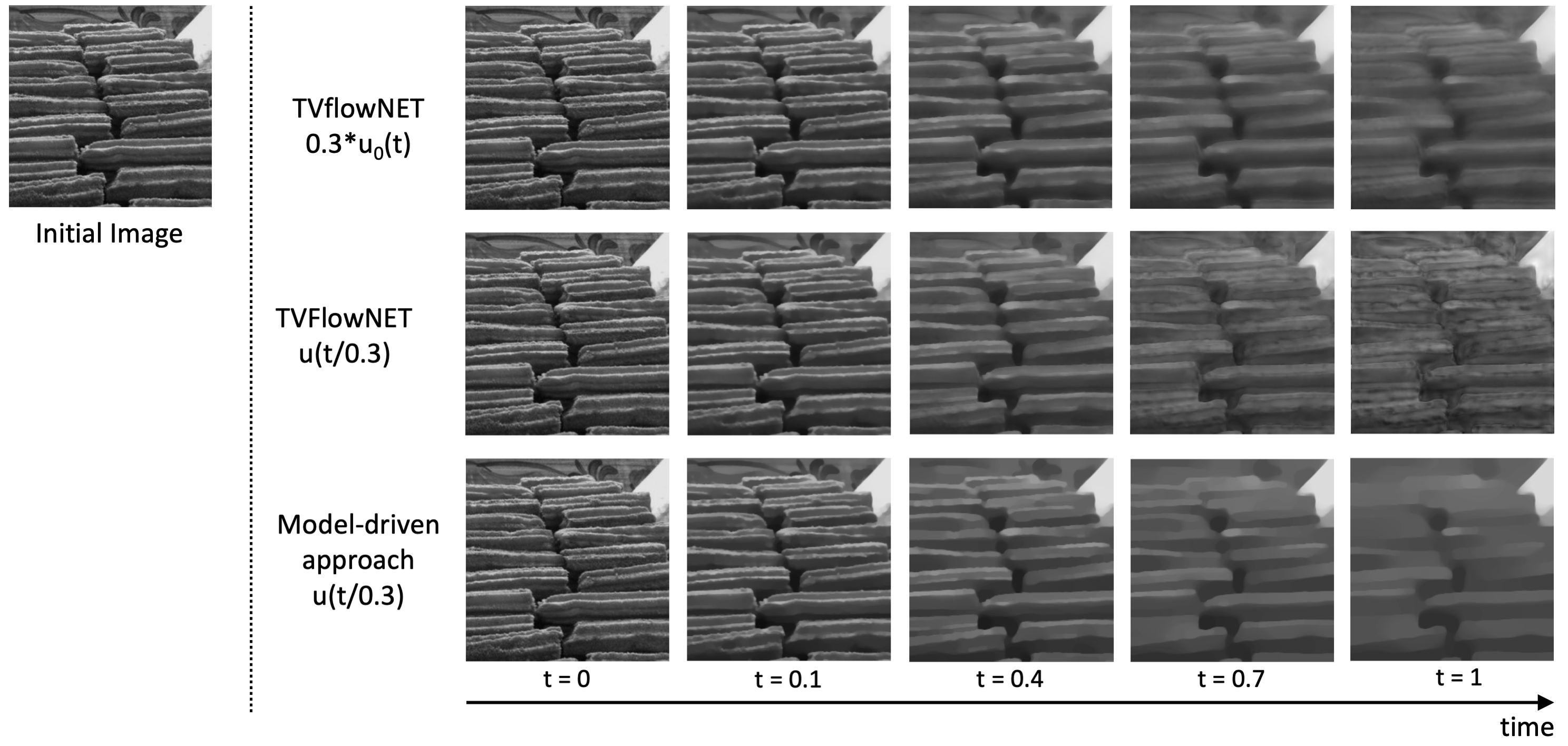}
    \caption[Visualisation of the one-homogeneity property for a factor $c=0.3$.]{Visualisation of the one-homogeneity property for a factor $c=0.3$ to display the effect of contrast change of the initial image $0.3u_0(t)$ and the associated change in time $u(t/0.3)$. Results are compared to the time change in the model-driven approach. The example image is taken from the Food101~\cite{Food101} dataset.}
    \label{fig.CH5.TVFlowNET_result_onehom}
\end{figure}

\subsubsection{Spectral TV decomposition}
\begin{table}[t]
\centering
    \begin{tabular}{ccccc}
     \toprule
    \textbf{Architecture} & \textbf{Training Size} & \multicolumn{3}{c}{\textbf{$u_0$ reconstruction}}                   \\
          & & Autodiff & FD TVflowNET & FD GT  \\
     \hline
     Semi-ResNet & 32 x 32px & 47.53 (0.997) & 47.68 (0.997) & 66.53 (0.999)\\
     Semi-ResNet & 96 x 96px & 43.82 (0.994) & 44.18 (0.995) & 66.53 (0.999)\\
     Semi-ResNet & 256 x 256px & 29.98 (0.877) & 30.21 (0.885) & 66.53 (0.999)\\
     Semi-ResNet & mixed & 41.35 (0.988) & 41.72 (0.988) & 66.53 (0.999)\\
     U-Net & 32 x 32px & \textbf{50.32} (0.998) & \textbf{50.32} (0.998) & 66.53 (0.999)\\
     U-Net & 96 x 96px & 48.14 (0.997) & 48.2 (0.997) & 66.53 (0.999)\\
     U-Net & 256 x 256px & 38.57 (0.971) & 38.58 (0.971) & 66.53 (0.999)\\
     U-Net & mixed & 43.94 (0.992) & 43.96 (0.992) & 66.53 (0.999)\\
     LGD & 32 x 32px & 49.69 (0.998) & 49.71 (0.998) & 66.53 (0.999)\\
     LGD & 96 x 96px & 45.79 \textbf{(0.996)} & 45.82 \textbf{(0.996)} & 66.53 (0.999)\\
     LGD & 256 x 256px & 37.14 (0.969) & 37.15 (0.97) & 66.53 (0.999)\\
     LGD & mixed & 42.82 (0.993) & 42.84 (0.993) & 66.53 (0.999)\\
     \bottomrule
    \end{tabular}
    \caption[Comparison of the initial image reconstruction using automatic differentiation versus finite differences of the TVflowNET as well as the ground truth results.]{Comparison of the initial image reconstruction using automatic differentiation (Autodiff) versus finite difference (FD) of the TVflowNET as well as the ground truth (GT) results at 50 time instances. The reported PSNR (SSIM) values are based on the comparison to the initial image and correspond to the average over $200$ images of the Food101~\cite{Food101} testing dataset.\label{tb.CH5.TVflowNET_results_autodiff}}
\end{table}
The motivation for learning the TV flow solution initially stemmed from the need to generate faster methods for spectral TV decomposition, where solving the TV flow is the most computationally expensive part. In the spectral TV decomposition process, the subsequent step to solving the TV flow involves computing the TV transform $u_{tt}(t;x)t$, which requires evaluating the second temporal derivative of the TV flow solution $u$. However, our approach did not explicitly penalise the training of the TVflowNET to align the automatic differentiation second derivative with the actual second derivative of the function. In this section, we examine the use of automatic differentiation versus finite differences on the trained neural networks for calculating the spectral TV decomposition. We demonstrate that the TVflowNET solution can be utilised to acquire the decomposed images. Specifically, we assess the reconstruction of the initial images from band-pass filtered images. We know that the integral over time of the TV transform renders the initial image~\cite{Gilboa2013,Gilboa2014}. That is, the inverse transform is defined as
\begin{align*}
    u_0(x) = \int_0^T u_{tt}(t;x)t \, dt + u_r(T,x),
\end{align*}
where $u_r(T,x) = u(T;x) - u_t(T;x)T$ is the residual~\cite{Gilboa2014}. Our comparison for the recovery of $u_0$ is threefold: Firstly, we derive the TV transform based on the TVflowNET output using automatic differentiation and secondly, using finite differences. Lastly, we derive the TV transform from the ground truth images generated by the model-driven approach and using finite differences. We then integrate over all time instances and compare the resulting images to the initial image. The results are shown in Table \ref{tb.CH5.TVflowNET_results_autodiff} based on $200$ images of size $96 \times 96$px and $50$ time instances. The performances of the derivative via automatic differentiation and finite differences based on the TVflowNET output are almost identical. As expected, the initial image reconstruction using the ground truth TV flow solution performs best. However, all three approaches have high PSNR and SSIM values, and the TVflowNET solution is, therefore, appropriate for deriving the spectral TV decomposition.

\subsubsection{Computation Time}
\begin{table}[t]
\centering
    \begin{tabular}{cccc}
     \toprule
    \textbf{Architecture} & \textbf{Training Size} & \textbf{Training (min)} & \textbf{Time per Epoch (sec)}                 \\
     \hline
     Semi-ResNet & 32 x 32 & 721.8 & 29.0 \\
     Semi-ResNet & 96 x 96 & 755.2 & 31.6\\
     Semi-ResNet & 256 x 256 & 2106.6 & 281.5\\
     Semi-ResNet & mixed & 806.5 & 40\\
     U-Net & 32 x 32 & 860.5 & 36.8 \\
     U-Net & 96 x 96 & 902.3 & 41.5\\
     U-Net & 256 x 256 & 2158.1 & 177\\
     U-Net & mixed & 1088.5 & 56\\
     LGD & 32 x 32 & 965.5 & 44.5\\
     LGD & 96 x 96 & 1069.7 & 47.4\\
     LGD & 256 x 256 & 2148.1 & 346.3\\
     LGD & mixed & 947.4 & 61.9\\
     \bottomrule
    \end{tabular}
    \caption[Training times in minutes of the different TVflowNET architectures trained on varying image sizes.]{Training times in minutes of the different TVflowNET architectures trained on varying image sizes. This refers to the time to reach the model with the best results on the validation set. Additionally, the time per epoch is reported in seconds.\label{tb.CH5.TVFlowNET_training_time}}
\end{table}
We have demonstrated the possibility of accurately approximating the TV flow solution using the TVflowNET in much detail above. The primary objective of employing a neural network approximation, next to achieving high image quality, is to increase computational efficiency and reduce processing time. To assess the computational speedup, we compare the evaluation time of each TVflowNET architecture with both the model-driven approach and the joint space-time optimisation. We also include the training time in Table \ref{tb.CH5.TVFlowNET_training_time} to provide a comprehensive analysis. All TVflowNET architectures were trained on an Nvidia A100 GPU and evaluated on an Nvidia P6000 GPU. The reported times correspond to GPU processing time. For the model-driven approach, we consider the computation time on both a CPU and an Nvidia P6000 GPU, whereas the joint space-time optimisation is conducted on the same GPU.

First, let us examine the training time of the TVflowNET models in Table \ref{tb.CH5.TVFlowNET_training_time}. The Semi-ResNet, comprising $14,947$ parameters, exhibited an average training time of $18.3$ hours. In comparison, the U-Net, with $272,972$ parameters, required $20.9$ hours on average for training. The LGD, containing $33,415$ training parameters, required the longest average training time of $21.4$ hours. The models trained on $256 \times 256$px images did not finish running through $750$ epochs in the 36-hour limit that is dictated by the computing resources. When inspecting the time per epoch for the other image sizes, we can see that the Semi-ResNet allows for the fastest training time. As it has the lowest number of parameters, this was to be expected. However, training the U-Net is slightly faster than training the LGD despite having $10$ times the number of parameters. We attribute this behaviour to the difference in network design where the LGD emulates an optimisation approach and includes multiple small networks. 

\begin{table}[t]
\centering
    \begin{tabular}{ccccc}
     \toprule
    \textbf{Architecture} & \multicolumn{4}{c}{\textbf{Evaluation Time (sec)}}                   \\
          & 32 x 32 & 96 x 96 & 256 x 256 & 512 x 512  \\
     \hline
     Semi-ResNet & 0.001 & 0.005 & 0.03 & 0.12\\
     U-Net & 0.003 & 0.01 & 0.05 & 0.19\\
     LGD & 0.008 & 0.014 & 0.073 & 0.29\\
     Model-driven (CPU) & 0.58 & 5.61 & 44.07 & 89.22\\
     Model-driven (GPU) & 3.12 & 3.14 & 3.17 & 3.42\\
     Joint space-time optimisation & 14.09 & 14.39 & 581.87 & 3217.93\\
     \bottomrule
    \end{tabular}
    \caption[Evaluation times in seconds of the TVflowNET, model-based approaches and joint space-time optimisation for different image sizes.]{Evaluation times in seconds of the TVflowNET, model-based approaches and joint space-time optimisation for different image sizes. For the latter approach, the optimisation was run for $2,000$ epochs for image sizes $32 \times 32$ and $96 \times 96$, while for images of size $256 \times 256$ it was run for $14,000$ epochs and for the largest image size $512 \times 512$ for $20,000$ epochs to obtain good quality results.\label{tb.CH5.TVFlowNET_evaluation_time}}
\end{table}

In comparing all three solution approaches, a noticeable difference in their treatment of the time component becomes apparent. The model-driven approach involves the evolution of the flow through time and relies on the TV flow solution from the previous time step. On the other hand, the joint space-time optimisation method obtains results for multiple time instances simultaneously. Still, it requires solution approximations for neighbouring time instances for the finite difference calculation of the temporal derivate. In contrast, the TVflowNET approach does not depend on solving the TV flow at any previous time instances but rather evaluates the solution $u(t)$ solely at the specific time of interest, allowing flexibility and significantly faster evaluation. Furthermore, the TVflowNET is capable of simultaneously processing multiple images. 

In Table \ref{tb.CH5.TVFlowNET_evaluation_time}, we present the results for evaluation time in seconds. As anticipated, the joint space-time optimisation exhibits the longest computation time, steeply increasing with larger image sizes. It is worth noting that the number of iterations required for the optimisation also increases with the image size to achieve high-quality solutions. We report the times for the sake of completeness. 

Comparing the model-driven approaches on CPU and GPU, the GPU implementation becomes advantageous for images larger than $32 \times 32$px when the CPU implementation becomes exceedingly slower. The increase in computation time for the GPU implementation, in turn, is relatively incremental for the tested image sizes. In the case of TVflowNET architectures, the evaluation time remains unaffected by the size of training images; hence, we report results for Semi-ResNet, U-Net, and LGD in general. Regardless of the image sizes tested, the TVflowNET exhibits computation times at least ten times faster than the model-driven approach. The Semi-ResNet and U-Net evaluations are in almost every tested image size at least two orders of magnitude faster than the model-driven approach. The evaluation of the LGD is the slowest among the network architectures considered. Nevertheless, it remains at least one order of magnitude faster than the model-driven approach and even two orders faster for smaller images. Overall, the TVflowNET yields a remarkable improvement in computational speed.

\section{Conclusion}
The numerical solution of the TV flow using model-driven approaches is cumbersome and computationally expensive. In this paper, we have introduced TVflowNET, an unsupervised deep learning approach for approximating the solution of the TV flow given an initial image and a time instance. Our approach achieves a significant computational speedup of up to two orders of magnitude compared to model-driven approaches. The TVflowNET does not rely on ground truth data but is rather based on the PDE itself. We designed a novel loss functional that circumvents the challenges related to the instability of the subgradient of the TV functional. This is achieved by simultaneously learning the TV flow solution and the diffusivity term of the TV flow. We highlighted the feasibility of the loss functional to approximate the TV flow solution by optimising without the use of neural networks. The joint space-time optimisation showed numerically that minimising the loss functional indeed recovers the TV flow solution even without a deep learning framework, however, at a much slower computation time. The high-quality results of the optimisation demonstrate the feasibility of the designed loss for solving the TV flow.

We compared three different architectures for the neural network design: the Semi-ResNet, the U-Net and the LGD. To evaluate the performance of the TVflowNET with these architectures, we carried out an extensive analysis encompassing changing training regimes, evaluation of different image sizes and types to show the generalisability, and the investigation of intrinsic properties of the TV flow such as the behaviour of eigenfunctions and the one-homogeneity. We trained all network architectures on images of different sizes varying from $32 \times 32$px to $256 \times 256$px as well as mixed-size images. 

In terms of generalisability, all network architectures and all training sizes apart from the $256 \times 256$px training were able to successfully approximate the TV flow solution for images of sizes different from the training. In our experiments with varying types of images (e.g., those of discs and ellipses, as well as medical images, all which were not present in the training set), all TVflowNET architectures exhibited outstanding performance, underlining the generalisability of the network. The TVflowNET learned the intrinsic properties of the TV flow solution related to the eigenfunctions, such as disks, without the explicit input in the loss or training approach. 

The one-homogeneity property of the subdifferential of the TV flow dictates that a change in contrast of the initial image is equivalent to a change in time evolution for the TV flow solution. The results show that one-homogeneity is preserved in one direction, that is, for contrast changes that constitute a change in time parameter within the trained time interval. However, the results are less accurate for times outside the trained interval. Among the different network architectures, the U-Net trained on $32 \times 32$px images exhibited the best results with the highest PSNR and SSIM. While the one-homogeneity is not fully preserved in the other direction, this could be solved by training the network for larger time instances. 

Lastly, we compared the computation time of the TVflowNET to the model-driven approaches run on a CPU and a GPU, as well as the joint space-time optimisation. In terms of the evaluation time of the TVflowNET compared to the model-driven approach, we showed that we gained an improvement of around two orders of magnitude and even higher for some image sizes. This is a significant advantage as it allows for a rapid calculation of the TV flow solution. Additionally, the TVflowNET does not rely on the evolution through all times but rather evaluates the solution at a time instance of interest. In that, it is cutting out unnecessary steps and making the derivation even faster. Specifically for large images or when processing a large number of images, this will result in a significant difference and enable not only faster processing but also more flexible use in terms of the time instances evaluated. 

The findings we presented here were given in the form of an ablation study that compared the different architectural designs. Although the majority of the network architectures yielded comparable results across the evaluated properties and images, the U-Net trained on $32 \times 32$px images exhibited the best overall performance. In particular, out of all networks considered, this model demonstrated the best performance in relation to the one-homogeneity condition and the effect of contrast changes that result in a time shift beyond the trained interval. However, it should be noted that training the U-Net requires a considerable amount of time despite yielding very similar outcomes to other architectures. As a close second, the LGD trained on $32 \times 32$px images equally shows outstanding results in terms of generalisability and learned intrinsic properties. The code for the TVflowNET is publically available on GitHub\footnote{\url{https://github.com/TamaraGrossmann/TVflowNET}}.

For future work, we are interested in investigating the learned subgradient in more detail. In the current setting, we have tuned the parameters to render the best results for the TV flow solution. However, the performance for the learned subgradient was lower, which raises the question of whether the learning approach can be amended to recover both the solution operator and the subgradient with equally high accuracy.

\section*{Acknowledgments}
The authors acknowledge the EPSRC grant EP/V026259/1  and the support of the Cantab Capital Institute for the Mathematics of Information. TGG additionally acknowledges the support of the EPSRC National Productivity and Investment Fund grant Nr. EP/S515334/1 reference 2089694. SD is grateful for the support of the AIX-COVNET collaboration which is funded by Intel and the EU/EFPIA Innovative Medicines Initiative project DRAGON (101005122). YK acknowledges the support of the EPSRC (Fellowship EP/V003615/2) and the National Physical Laboratory. CBS acknowledges support from the Philip Leverhulme Prize, the Royal Society Wolfson Fellowship, the EPSRC advanced career fellowship EP/V029428/1, EPSRC grants EP/T003553/1, EP/N014588/1, the Wellcome Trust 215733/Z/19/Z and 221633/Z/20/Z, Horizon 2020 under the Marie Skodowska-Curie No. 777826 NoMADS and the Alan Turing Institute.
We also acknowledge the support of NVIDIA Corporation with the donation of two Quadro P6000 GPUs used for this research.

\bibliographystyle{abbrv}
\bibliography{bibliography}
\end{document}